\documentclass{article}
\PassOptionsToPackage{numbers, compress}{natbib}

\usepackage[preprint]{neurips_2026}

\usepackage[most]{tcolorbox}
\usepackage{hyperref}
\usepackage{url}
\usepackage{caption}
\usepackage{graphicx} 
\usepackage{enumitem}
\usepackage{booktabs}
\usepackage{multirow}
\usepackage{makecell}
\usepackage{graphicx}
\usepackage{subcaption}
\usepackage{algorithm}
\usepackage{algpseudocode}
\usepackage{amsmath, amssymb, amsthm}
\newtheorem{lemma}{Lemma}

\newtheorem{corollary}{Corollary}
\newtheorem{proposition}{Proposition}
\NewDocumentCommand{\heng}
{ mO{} }{\textcolor{red}{\textsuperscript{\textit{Heng}}\textsf{\textbf{\small[#1]}}}}

\NewDocumentCommand{\thao}
{ mO{} }{\textcolor{blue}{\textsuperscript{\textit{thao}}\textsf{\textbf{\small[#1]}}}}

\title{MolLingo: Molecule-Native Representations for LLM-Powered Scientific Agents}

\author{
 Thao Nguyen\textsuperscript{1} \quad
 Heng Ji\textsuperscript{1} \\
 \textsuperscript{1}Siebel School of Computing and Data Science, University of Illinois Urbana-Champaign \\
 \texttt{\{thaotn2, hengji\}@illinois.edu}
}

\begin{document}

\maketitle



\begin{abstract}

We present \textbf{MolLingo}, a multi-agent system that emulates the reasoning process of a chemist to automate molecular design. Existing LLM-based approaches to molecular design either operate as standalone generative models without access to external tools, or lack the multi-agent coordination and shared memory needed to support iterative, evidence-driven reasoning across the full molecular design pipeline.
MolLingo addresses this by coordinating a Literature Agent, a Chemist Agent, and an Orchestrator through a shared memory module, where each agent is equipped with domain-specific tools.
To enable effective molecular reasoning within this multi-agent framework, we introduce BRICS-based Fragment Enumeration (BFE). BFE is a synthesis-aware molecular fragmentation method that decomposes molecules into chemically meaningful building blocks, represented as block-based SMILES paired with common chemical names. This representation bridges the gap between molecular structure and LLM semantic space, enabling block-level reasoning and editing that is intractable with raw SMILES.
As a case study in early-stage therapeutic design, MolLingo further grounds the Chemist Agent's reasoning in binding site geometry and residue-level protein context derived from molecular docking, to optimize molecules for stronger target binding. Across four benchmarks, MolLingo consistently outperforms frontier LLMs and specialized baselines: a fourfold docking score improvement over GPT-5.4 despite sharing the same underlying model, consistent drug property optimization gains across multiple LLM backbones where the same models yield negative improvement on raw SMILES, and state-of-the-art results on TOMG-Bench, surpassing both frontier LLMs and RePO, an RL-based molecular optimization method. Our results suggest that LLMs are already capable molecular design assistants, they simply need to be probed in the right language, through chemically meaningful representations and biologically grounded structural context.
Code is available at 
\url{https://anonymous.4open.science/status/MolLingo-7450}.

\end{abstract}
\vspace{-1em}
\section{Introduction}
\begin{figure}[!hbt]
 \centering
 \includegraphics[width=0.9\linewidth]{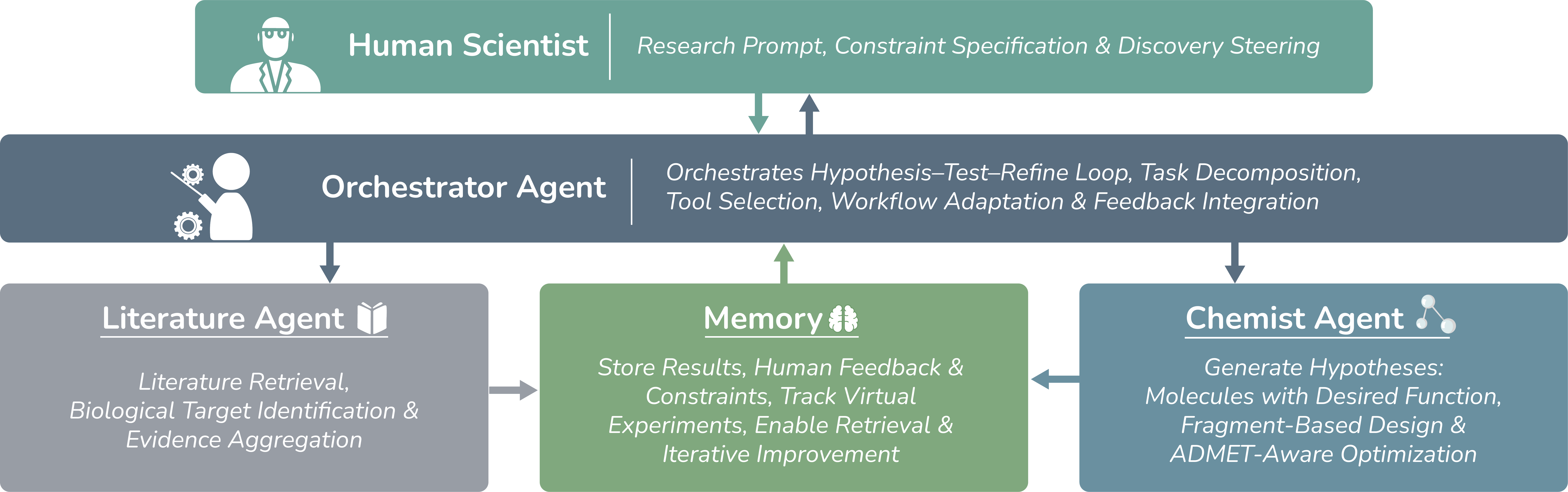}
 \caption{MolLingo agent-based architecture.}
 \label{fig:teaser}
 
\end{figure}

Molecular design is central to progress in medicine, clean energy, and advanced materials, and lies at the heart of addressing many of today's global challenges.
It requires autonomous systems capable of integrating heterogeneous biological and physical data, chemical constraints, and insights from the scientific literature. Traditional computational pipelines often fall short in this regard, as they are typically designed as static, one-directional workflows that lack the adaptability needed to incorporate new data or expert feedback~\citep{sadybekov2023computational}. 

To address these limitations, we introduce MolLingo, an autonomous multi-agent system designed to function as an expert collaborator in the molecular design process. By coordinating an Orchestrator Agent for workflow management, a Literature Agent for retrieving scientific knowledge, and a Chemist Agent for LLM-based molecular reasoning through a shared memory module that persists knowledge and reasoning traces across the discovery workflow, MolLingo provides a goal-driven framework capable of navigating complex chemical, biological and physical spaces with minimal human intervention.

A key strength of MolLingo lies in its ability to leverage the emergent chemical intuition of LLMs. Trained on vast corpora of scientific text, models such as GPT~\citep{achiam2023gpt} exhibit reasoning patterns that resemble those of chemists~\citep{boiko2023autonomous, bran2023chemcrow}. However, a critical bottleneck remains in the representation of molecular structures. Most existing approaches rely on SMILES (Simplified Molecular Input Line Entry System) strings~\citep{weininger1988smiles}, which are often lengthy, syntactically fragile, and structurally opaque to transformer-based architectures~\citep{guo2023can, ucak2023reconstruction}. Furthermore, while the space of possible molecule-level SMILES strings is effectively unbounded
, the space of chemically meaningful building blocks such as functional groups, ring systems, and pharmacophoric fragments is finite and limited. As a result, block-level entities appear far more frequently in LLM pretraining corpora than full molecule SMILES, allowing LLMs to develop stronger parametric knowledge at the building block level.

In light of this, MolLingo introduces an alternative molecular representation to raw SMILES. The system decomposes molecular graphs into discrete, chemically meaningful building blocks, each represented as a common chemical name paired with its block-level SMILES string (Fig.~\ref{fig:tokenizer}c). This representation directly leverages two properties of LLM parametric knowledge: semantic understanding of named chemical entities, and prior knowledge with short, syntactically simple SMILES fragments.
These block-based representations enable the LLM to reason at a higher level of abstraction, treating molecular design as \textit{the composition and modification of functional units}. This paradigm aligns more naturally with the conceptual language of chemistry~\citep{sutherland2008chemical} and shows improved performance compared to conventional generative models trained on raw SMILES sequences~\citep{unlu2025target}.
Equipped with this enhanced representation and a task-specific toolkit, MolLingo supports end-to-end molecular design workflows, demonstrated here through drug design, spanning biological target identification, hit discovery, hit-to-lead progression, and lead optimization (Fig.~\ref{fig:dd_pipeline}). Notably, for hit-to-lead progression, MolLingo further grounds the Chemist Agent's reasoning in three-dimensional structural context derived from molecular docking, enabling it to propose targeted structural modifications guided by the geometry and chemical environment of the target binding site.

Unlike pipeline-based approaches, MolLingo operates as a dynamic and interactive system rather than a fixed sequence of steps. The agents are coordinated through a shared memory module that maintains intermediate results, design rationales, and evolving objectives across the discovery process. This memory enables iterative refinement, allowing the system to revisit earlier decisions, incorporate newly retrieved knowledge, and adapt its strategy based on emerging evidence. Furthermore, MolLingo supports human-in-the-loop interaction, enabling domain experts to guide the discovery process by providing feedback, constraints, or design goals at any stage. Empirical results across four benchmarks demonstrate that providing LLMs with the right molecular representation and molecular design context is key to unlocking their potential as chemistry assistants.

\section{The MolLingo Framework}
\vspace{-0.5em}

This section describes the key components of MolLingo. We first present the overall architecture and agent coordination mechanism (\S\ref{sec:architecture}), followed by the block-based representation that equips LLMs with structure–function awareness (\S\ref{sec:smiles}). We then use drug design as a case study and detail the LLM-guided molecular design pipeline, covering hit identification, hit-to-lead progression, and lead optimization (\S\ref{sec:LLM-guided Molecular Design}).

\vspace{-0.5em}
\subsection{The MolLingo Architecture}
\label{sec:architecture}
 MolLingo coordinates three specialized agents: an Orchestrator, a Literature Agent, and a Chemist Agent, through a centralized shared memory module that persists all extracted knowledge, intermediate results, and reasoning traces across the discovery workflow. This architecture enables each agent to build upon prior findings, revisit earlier decisions, and adapt its strategy as new evidence accumulates, transforming MolLingo from a static pipeline into a dynamic, evidence-driven system.

\textbf{The Shared Memory Module} serves as the coordination hub of the MolLingo architecture. It maintains structured information (e.g., candidate molecules, docking scores, target annotations) and unstructured context (e.g., reasoning traces, literature insights, design rationales), making them accessible to all agents at every stage. Rather than passing outputs sequentially between components, agents read from and write to this shared store, functioning as a long-term memory that accumulates knowledge across the full discovery workflow and enables iterative refinement and coherent multi-step reasoning. A description of the memory schema, retrieval mechanism, and lifecycle management is provided in Appendix~\ref{app:memory}.

\textbf{The Orchestrator Agent} governs the overall workflow and coordinates interactions among specialized agents, external tools, and the shared memory module. Given an initial design objective, such as a disease target or optimization goal, it decomposes the task into subgoals (e.g., target identification, hit discovery, lead optimization), selects the appropriate agent at each stage, and integrates outputs into the shared memory for downstream use. The Orchestrator operates in a dynamic, goal-driven manner, adapting the workflow based on intermediate results and newly acquired knowledge. It also supports human-in-the-loop interaction, allowing domain experts to provide feedback or high-level constraints at any stage, which are incorporated into subsequent planning steps. The decision-making algorithm and prompt template of the Orchestrator are detailed in Appendix~\ref{app:orchestrator}.

\textbf{The Literature Agent} is an LLM that employs retrieval-augmented generation to retrieve and synthesize domain knowledge from scientific literature and curated biomedical sources, including PubMed, UniProt, and ChEMBL. Given a task or target of interest (for example, identifying the biological target for a given disease), it identifies relationships between diseases, genes, proteins, pathways, and known bioactive compounds, distilling them into structured insights that guide target selection, hit discovery, and molecular optimization. All extracted knowledge is written to the shared memory module, where it informs the reasoning of both the Orchestrator and Chemist Agent throughout the workflow.

\textbf{The Chemist Agent} is responsible for molecular reasoning and compound design across the drug discovery pipeline. For \textbf{hit discovery}, it retrieves candidate binders from protein–ligand binding databases such as ChEMBL~\citep{gaulton2012chembl}, or performs fragment screening using DrugCLIP~\citep{gao2023drugclip} when experimental data is unavailable. In the \textbf{hit-to-lead} stage, it performs \textit{context-aware fragment growing} by leveraging structural information from protein–ligand docking to propose chemically valid modifications at specific atomic positions. During \textbf{lead optimization}, it iteratively refines candidate molecules to improve binding affinity and drug-like properties, guided by docking evaluations and a suite of property assessment tools including property prediction (Absorption, Distribution, Metabolism, Excretion, and Toxicity), synthesizability estimation, and drug-likeness scoring. Throughout this process, the Chemist Agent reads from and writes to the shared memory, grounding its LLM-based reasoning in accumulated structural and biological context.
Fig.~\ref{fig:teaser} provides an overview of the MolLingo architecture.

\vspace{-0.5em}
\subsection{Block-based SMILES for Structure–Function Awareness}
\label{sec:smiles}
A fundamental challenge in applying LLMs to molecular design is the mismatch between SMILES and the semantic space of language models. Standard SMILES strings encode molecular structure as a flat character sequence, producing representations that are syntactically fragile, structurally opaque, and misaligned with the chemical concepts encoded in LLM pretraining corpora. A single character substitution can render a molecule invalid, and functionally meaningful substructures, such as aromatic rings or pharmacophoric groups, are scattered across non-contiguous character spans with no explicit boundaries (Fig.~\ref{fig:tokenizer}a,b).

To address this, we introduce a block-based representation that decomposes molecular graphs into discrete, chemically meaningful building blocks, producing sequences that are better aligned with the semantic space of LLMs. A molecule is represented as an ordered sequence of functional units (such as functional groups, linkers, and pharmacophoric fragments) each expressed as a block-level SMILES string paired with its common chemical name (Fig.~\ref{fig:tokenizer}c). This representation is constructed in three steps:

\textbf{(1) Vocabulary Construction via BRICS-based Fragment Enumeration (BFE).} We first train a tokenizer over a large corpus of drug-like molecules using BRICS-based Fragment Enumeration, detailed in Appendix~\ref{sec:bfe}. As a first step, we apply BRICS~\citep{degen2008art} to decompose each molecule into small, synthetically accessible fragments by breaking at chemically valid single bonds. These BRICS fragments, rather than individual atoms, then serve as the primitive units of our vocabulary, analogous to characters in text BPE. Starting from these fragment primitives, the standard approach for graph-based vocabulary construction iteratively merges the most frequent adjacent fragment pairs across the corpus until a target vocabulary size is reached. However, this requires $(V - n_0)$ full corpus passes, where $V$ is the target vocabulary size and $n_0$ is the initial atomic vocabulary size, with expensive graph merge operations at each step, making it costly at the scale of 20M molecules. BFE instead enumerates all valid fragments in a single corpus pass by exhaustively breaking each molecule at pairs of BRICS bonds, achieving a $O\big((V - n_0)\,\bar{n}\big)$ speedup over iterative graph-based construction, where $\bar{n}$ is the average number of blocks per molecule (see Appendix~\ref{sec:bfe} for the full complexity analysis). The result is a vocabulary of chemically recurring substructures that reflect the structural regularities of drug-like chemical space, with the guarantee that all building blocks correspond to synthetically accessible fragments.

\textbf{(2) Molecular Fragmentation.} Given a new molecule, we decompose it into a sequence of building blocks, using the learned vocabulary to guide the selection of the optimal decomposition. We first identify all BRICS bonds in the molecule and enumerate all possible ways to break it into non-branching sequences of blocks, yielding a set of candidates $\mathcal{T}$. Each candidate is a sequence $[t_1, \ldots, t_k]$ where each $t_i$ is a building block. Since a molecule can be decomposed in multiple ways (for example, $[t_1 t_2, t_3, t_4]$ or $[t_1, t_2 t_3, t_4]$), we select the optimal fragmentation hierarchically, prioritizing longer blocks first. Among all candidates, we retain those with the fewest blocks (i.e., longest blocks), as longer blocks carry richer structural and functional information. Among these, we keep only candidates whose blocks all appear in the vocabulary with frequency above a minimum threshold $f_{\min}$. If multiple candidates satisfy this condition, we select the one with the lowest standard deviation of block frequencies, favoring decompositions where blocks are uniformly well-represented in the vocabulary:
\[
t^* = \underset{t \in \mathcal{T}^*}{\arg\min}\ 
\text{std}\left(\{freq(t_i)\}_{i=1}^{k}\right)
\]
where $\mathcal{T}^*$ is the set of valid candidates satisfying both the minimum block count and frequency threshold criteria, and $freq(\cdot)$ is the corpus frequency of a block. The full fragmentation procedure is provided in Algorithm~\ref{alg:tokenization}.

\textbf{(3) Block-based SMILES with common names.}
Once the optimal decomposition is selected, each block is converted to 
a SMILES string with wildcard atoms at its attachment points. Since 
BRICS bonds are single bonds and each block has at most two neighbors 
in a non-branching decomposition, each block carries at most two 
wildcard atoms, labeled by isotope: \texttt{[1*]} for the left 
attachment point and \texttt{[2*]} for the right. The connection rule 
is simple and consistent: \texttt{[1*]} of each block always connects 
to \texttt{[2*]} of the next block and vice versa. Fig.~\ref{fig:tokenizer} 
provides an example of the full fragmentation pipeline applied to 
imatinib. The molecule is fragmented into blocks:
$\texttt{[2*]c1cccnc1}, \texttt{[2*]Nc1nccc([1*])n1}, \texttt{[2*]c1ccc(C)c([1*])c1},$
$\texttt{[1*]NC(=O)c1ccc([2*])cc1}, \texttt{[1*]CN1CCN(C)CC1}$.
Stripping the wildcard atoms from each block yields its parent 
scaffold, which can be mapped to a common chemical name: pyridine, 
2-aminopyrimidine, toluene, benzamide, and piperazine respectively. The 
final molecular representation thus takes the form of a sequence of 
block SMILES paired with their common names, for example:

\vspace{-0.5em}
\[
\texttt{pyridine} \rightarrow \texttt{2-aminopyrimidine} \rightarrow 
\texttt{toluene} \rightarrow \texttt{benzamide} \rightarrow 
\texttt{piperazine}
\]
\vspace{-0.5em}

Rather than presenting the model with a long, syntactically opaque SMILES string, this block-based representation decomposes the molecule into named, chemically meaningful units that are well-represented in LLM pretraining corpora, directly aligning molecular structure with the semantic space of language models.


\begin{figure}
\centering

\begin{minipage}{0.38\linewidth}
 \captionof{figure}{Molecule representation methods for LLMs. (a) The example molecule is fragmented into functional blocks, each contributing to specific chemical, physical or biological properties. (b) SMILES representation, used in most agentic drug discovery systems. Functional blocks are fragmented across the SMILES, limiting LLMs’ ability to link structure and function. (c) Our block-level SMILES representation encodes molecules as sequences of chemically meaningful blocks, labeled by both name and SMILES, enabling LLMs to reason over structure–function relationships effectively.}
 \label{fig:tokenizer}
\end{minipage}
\hfill
\begin{minipage}{0.6\linewidth}
 \centering
 \includegraphics[width=\linewidth]{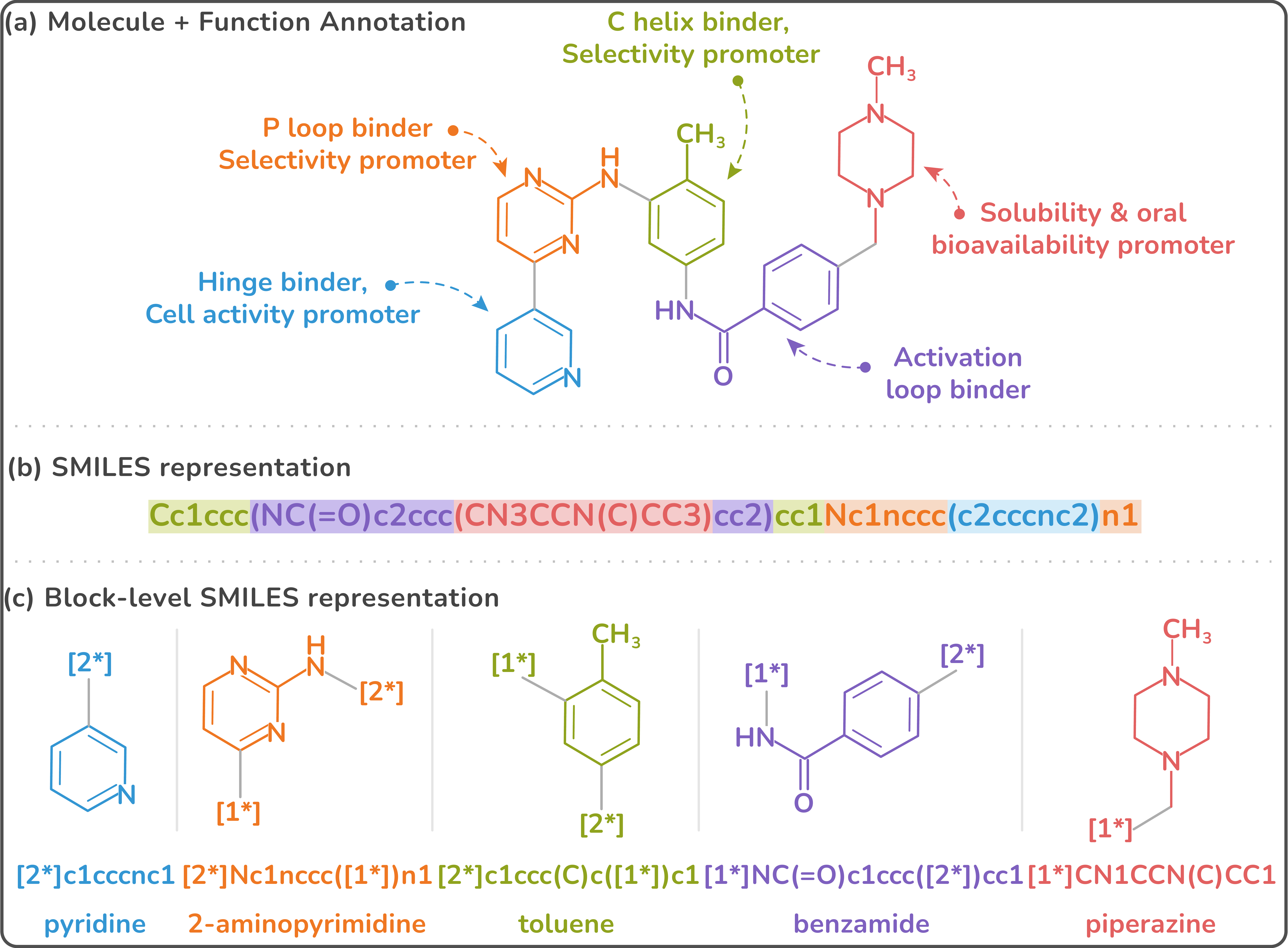}
\end{minipage}
\end{figure}



To validate that block-based representation genuinely improves structural alignment with LLM semantic space, we probe the attention mechanism of Qwen2-Instruct-7B~\citep{qwen2} under two input conditions: raw SMILES and block-based SMILES paired with common names. For the same molecule, we construct two sentences using the same template: a molecular representation followed by a description of a biological function known to be associated with a specific block, and compute the average attention between all tokens belonging to each block and the tokens describing that function. As shown in Fig.~\ref{fig:attention_probing}, when the input is raw SMILES, attention scores are distributed randomly across tokens with no meaningful structure–function correspondence. In contrast, block-based representation produces strongly localized attention: the highest attention is observed between the tokens of the relevant block and the words describing its associated function. This confirms that block-based representation enables LLMs to form explicit associations between molecular substructures and their chemical functions (a capability that is absent with raw SMILES input).

\vspace{-0.4em}
\subsection{LLM-guided Molecular Design}
\vspace{-0.5em}
\label{sec:LLM-guided Molecular Design}
 MolLingo approaches molecular design as a reasoning task grounded in chemistry and biology knowledge. Rather than treating molecular optimization as a purely statistical generation process, MolLingo leverages the Chemist Agent to reason over block-based representations, structural context from protein--ligand interactions, and accumulated knowledge from the shared memory module. This section describes these three stages of the drug design pipeline: hit identification, where candidate binders are retrieved or screened against a target protein (\S\ref{sec:hit}); hit-to-lead progression, where promising hits are expanded through context-aware fragment growing (\S\ref{sec:h2l}); and lead optimization, where candidate molecules are iteratively refined to improve binding affinity and drug-like properties (\S\ref{sec:lead}).

\vspace{-0.4em}
\subsubsection{Hit Identification}
\vspace{-0.5em}
\label{sec:hit}
The goal of hit identification is to find an initial set of molecules that bind to the target protein with sufficient affinity to serve as starting points for further optimization. MolLingo pursues two complementary strategies depending on data availability. When experimental binding data are available, the Chemist Agent queries protein--ligand binding databases such as ChEMBL~\citep{gaulton2012chembl} to retrieve known binders for the target protein. Retrieved candidates are ranked by binding affinity and drug-likeness score (QED~\citep{bickerton2012quantifying}), balancing potency with molecular quality before being passed to the hit-to-lead stage. When experimental data are unavailable (as is common for novel or understudied targets), MolLingo performs ligand-based virtual screening over a curated fragment library using DrugCLIP~\citep{gao2023drugclip}, a contrastive learning model that embeds proteins and ligands into a shared semantic space, enabling efficient similarity-based retrieval without explicit docking. The fragment library used for screening is described in Appendix~\ref{sec:fragment-library}. The top-ranked fragments are then forwarded to the hit-to-lead stage as initial hits for further elaboration.

\vspace{-0.4em}
\subsubsection{Hit-to-Lead: Context-Aware Fragment Growing}
\vspace{-0.5em}
\label{sec:h2l}

The hit-to-lead stage transforms an initial hit fragment into a more potent and drug-like lead compound by progressively growing the fragment within the binding site of the target protein. Unlike conventional fragment growing approaches that rely on exhaustive enumeration or purely data-driven generation, MolLingo performs context-aware fragment growing by combining structural information from molecular docking with the chemical reasoning capabilities of the Chemist Agent. This allows the system to propose chemically meaningful extensions at atomic positions where new fragments are most likely to form favorable interactions with the binding site, mimicking the reasoning process of a medicinal chemist who considers both the geometry of the pocket and the chemical nature of nearby residues when deciding where and how to grow a fragment.

Given a hit fragment $F$ and a protein target $P$, we first perform molecular docking to obtain the binding pose of $F$ within the active site of $P$, saving the resulting receptor and ligand coordinates as structures $R$ and $L$, respectively. We then analyze the docking pose to identify the most promising atomic positions for fragment growing, which we refer to as \textit{growth hotspots}. 
For each heavy atom $a_i \in \mathcal{A}_L$ of the docked ligand, we characterize two complementary properties of its local environment: the neighboring residues and the available space. Neighboring residues $\mathcal{R}_i$ are defined as all receptor residues containing at least one heavy atom within a contact threshold $d_c = 7.0$~\AA\ of $a_i$:
\[
\mathcal{R}_i = \{ r \in R \mid \exists\, a \in r : \|a_i - a\| \leq d_c \}.
\]

The available volume $V_i$ is estimated by constructing a cubic grid 
$\mathcal{G}_i$ centered at $a_i$ with edge length $5.0$~\AA\ and 
resolution $0.5$~\AA, and counting grid points that satisfy van der 
Waals clearance constraints with respect to both the receptor and the 
ligand:
\[
V_i = n_i \times (0.5)^3,
\]
where
$
n_i = \big|\{ g \in \mathcal{G}_i \mid d_R(g) > 2.2~\text{\AA} 
\;\wedge\; d_L(g) > 1.2~\text{\AA} \}\big|,
$
and $d_R(g)$, $d_L(g)$ are the minimum distances from grid point $g$ 
to any heavy atom in the receptor and ligand respectively. The threshold of $2.2$~\AA\ for the receptor approximates the van der Waals radius of a carbon atom ($\sim$1.7~\AA) plus a small probe radius, reflecting the minimum clearance required to avoid steric clash with receptor atoms~\citep{bondi1964van}. The threshold of $1.2$~\AA\ for the ligand is set to the approximate van der Waals radius of a hydrogen atom, ensuring that grid points immediately adjacent to the existing ligand are excluded while preserving attachment-accessible space at ligand periphery.

Together, the atom type ($\text{type} (a_i)$), available volume ($V_i$), and neighboring residues ($\mathcal{R}_i$) define the \textit{biological context} of each potential growth position. Hotspots are ranked by $V_i$ in descending order by available volume, and the top-$k$ candidates are retained.
For each hotspot, the biological context ($\text{type}(a_i)$, $V_i$, $\mathcal{R}_i$) is encoded into a natural language prompt and passed to the Chemist Agent. Rather than asking the LLM to propose molecule modifications given only the protein target, this design grounds the reasoning process in the three-dimensional structural environment of the binding site, turning the Chemist Agent into a spatially-aware reasoner that accounts for pocket geometry, steric accessibility, and the chemical nature of surrounding residues when proposing structural extensions. This enables the LLM to reason about fragment growing in a manner analogous to a medicinal chemist. The full hotspot identification and fragment growing pipeline is illustrated in Fig.~\ref{fig:molecule_growing}, and pseudocode is provided in Appendix~\ref{appendix:h2l}.
\begin{figure}[!hbt]
 \centering
 \includegraphics[width=\linewidth]{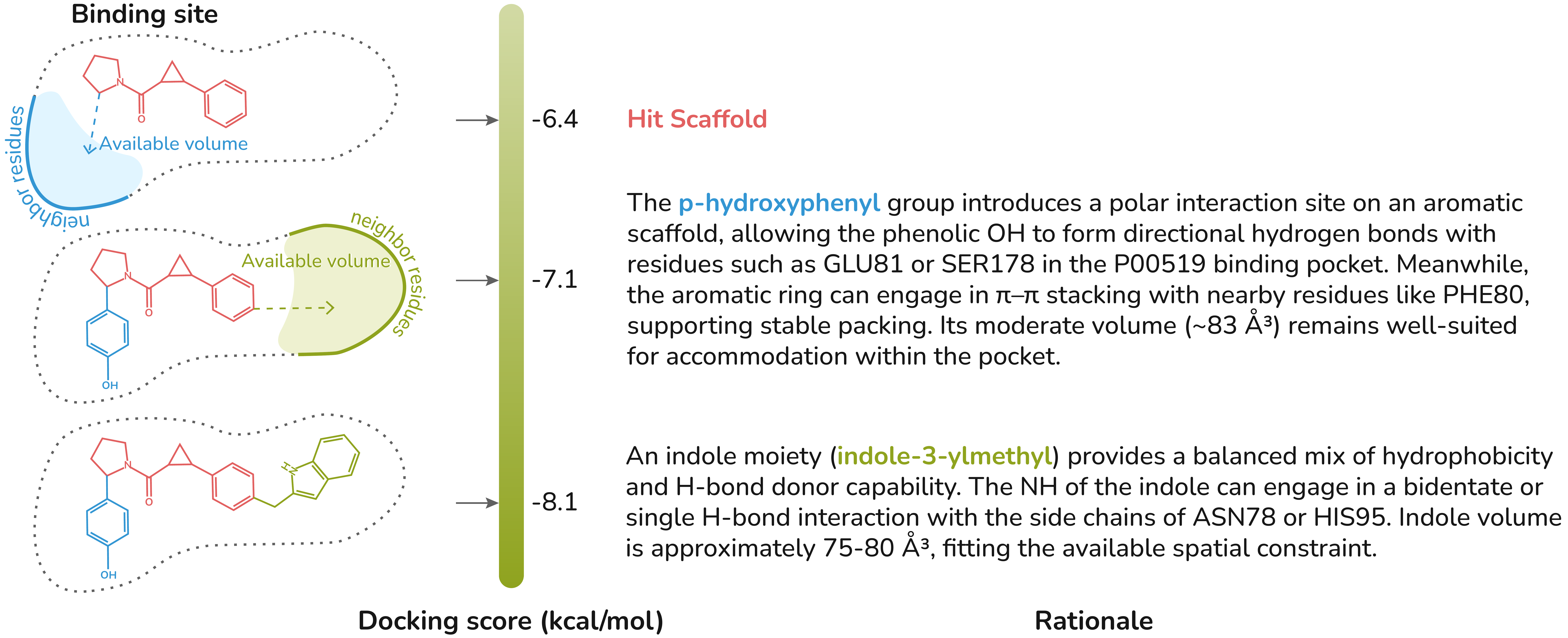}
 \caption{\textbf{Hit-to-Lead Optimization.} Starting from a hit scaffold and its docking pose with the target protein, the available volume within the binding site and neighboring amino acid residues are identified. Using this structural and biological context, the LLM reasons over fragment growth to iteratively expand the scaffold into a full molecule optimized for binding affinity.}
 \label{fig:molecule_growing}
\end{figure}

\subsubsection{Lead Optimization}
\vspace{-0.5em}
\label{sec:lead}

Lead optimization refines a promising lead compound into a drug candidate by improving its ADMET properties while preserving sufficient structural similarity to the original molecule to maintain binding affinity. MolLingo performs this iteratively through the Chemist Agent, which operates directly on the block-based representation introduced in \S\ref{sec:smiles}.
At each optimization step, the Chemist Agent is provided with two inputs: the block-based representation of the current lead molecule, and its ADMET profile as predicted by the oracle models described in Appendix~\ref{sec:admet-oracle}. Conditioned on this information, the agent reasons about which blocks are most likely responsible for unfavorable ADMET properties and proposes local modifications: replacing a building block with a chemically similar alternative, or substituting, adding, or removing individual functional groups within a block, to address the identified liabilities. By operating at the block level rather than the atom level, the agent naturally produces chemically meaningful and synthetically accessible modifications, while the block-based representation ensures that proposed changes remain interpretable and grounded in the LLM's chemical knowledge.

To maintain binding affinity throughout optimization, the Chemist Agent is instructed to propose conservative modifications that keep the optimized molecule sufficiently similar to the original lead, such that the binding mode to the target protein is likely preserved. Scaffold changes are permitted but discouraged unless clearly motivated by the ADMET profile, and the degree of structural deviation is monitored via molecular similarity metrics across iterations.
Each proposed molecule is evaluated by the ADMET oracle models, and the reasoning trace (including the identified liabilities, proposed modifications, and predicted outcomes) is stored in the shared memory module. This allows the Chemist Agent to build on prior reasoning steps in subsequent iterations, avoiding redundant modifications and progressively refining the molecule toward a favorable ADMET profile. After a fixed number of iterations, the molecule with the highest ADMET score is selected as the final drug candidate.

\vspace{-0.4em}
\section{Experimental Results}
\vspace{-0.5em}
\label{results}

We evaluate MolLingo across four benchmarks targeting distinct stages of the drug discovery pipeline. (i) \textbf{ADMET optimization}: we evaluate block-based molecular editing for toxicity reduction on the Withdrawn 2.0 database of clinically failed drugs, comparing against frontier LLMs operating on raw SMILES. (ii) \textbf{Hit-to-lead fragment growing}: we assess whether docking-guided biological context improves binding affinity over frontier LLMs, evaluated on 30 diverse protein targets. (iii) \textbf{Molecular property optimization on TOMG-Bench}: we compare against a broad set of LLMs and the specialized method RePO on LogP, QED, and MR optimization. (iv) \textbf{Hit identification against DrugGEN}: we compare retrieved hit quality against a GAN-based generative model on CDK2 and AKT1. Additionally, we evaluate the Literature Agent's biological target identification capability on 30 diverse diseases (results are reported in Appendix~\ref{app:target-identification}). Benchmarks (iii), (iv), and target identification are reported in Appendix~\ref{benchmarking} due to page limitations.

\vspace{-0.4em}
\subsection{Block-Based Representation for ADMET Optimization}
\vspace{-0.5em}
To evaluate the effectiveness of MolLingo's block-based molecular representation for lead optimization, we benchmark against a set of baseline models that operate on raw SMILES strings. All models are provided with the same input: the SMILES of the molecule to be optimized and its ADMET profile, and are tasked with proposing modifications to improve the target property while maintaining structural similarity to the original molecule. 

\textbf{Dataset.} We use the Withdrawn 2.0 database~\citep{gallo2024withdrawn}, a curated collection of 602 drugs withdrawn from the market with documented reasons for withdrawal. Among these, 91 are labeled as withdrawn due to DILI (drug-induced liver injury) and 33 due to hERG cardiotoxicity, providing a realistic and clinically grounded benchmark for ADMET optimization. To maximize the reliability of our oracle models as evaluation metrics, we first measure their true positive rate on these samples and restrict evaluation to molecules that are correctly identified by the oracle. This yields 67 out of 91 DILI cases (true positive rate: 73.6\%) and all 33 hERG cases (true positive rate: 100\%), on which we report optimization results, with no data overlap between these samples and the training DILI and hERG datasets (TDC’s).

\textbf{Baselines.} We compare MolLingo against a representative set of frontier LLMs prompted to perform molecular optimization directly on raw SMILES strings: GPT-5.4, Claude-4.6-Sonnet, Gemini-3-Pro, Qwen2-7B-Instruct, and Qwen2-14B-Instruct. These baselines are chosen to span a range of model scales and families, providing a comprehensive picture of how well general-purpose LLMs can perform molecular optimization without block-based representations. The prompt used for all baseline models is provided in Appendix~\ref{app:prompt}.

\textbf{Metrics.} We report five metrics: (i) Improvement (\%), the relative change in the target ADMET property score after optimization; (ii) Similarity (\%), the Tanimoto similarity between the original and optimized molecule, measuring how conservatively the model modifies the input; (iii) Drug-likeness, the predicted QED score; (iv) Synth, the estimated synthetic accessibility score; and (v) Validity (\%), the fraction of proposed molecules that are chemically valid. 

\textbf{Results.} Table~\ref{table:dili} reports results on the DILI benchmark. All three MolLingo variants consistently outperform their corresponding raw SMILES baselines by a large margin, most strikingly for Gemini-3-Pro: the raw model yields $-46.6\%$ improvement, while MolLingo (Gemini-3-Pro base) achieves the best overall result of $18.5\%$ improvement and $45.5\%$ similarity, with an Imprv$\times$Sim score of 839.8. The same pattern holds for GPT-5.4 ($-32.1\%$ raw vs. $13.0\%$ with MolLingo) and Claude-4.6-Sonnet ($9.5\%$ raw vs. $10.2\%$ with MolLingo at substantially higher similarity: $43.1\%$ vs. $16.9\%$). This cross-model consistency strongly suggests that the benefit of block-based representation is model-agnostic rather than specific to any particular backbone. Baselines operating on raw SMILES either yield negative improvement or achieve modest gains at low similarity, confirming that unconstrained SMILES generation tends to drift toward chemically dissimilar molecules that do not preserve the pharmacophore. MolLingo achieves 100\% validity across all variants, confirming that block-based operations naturally preserve chemical validity without post-hoc filtering. Results on the hERG benchmark follow a similar trend and are reported in Appendix~\ref{app:herg-results}.


\begin{table*}[!hbt]
\centering
\caption{DILI Comparison of model performance across multiple chemical generation metrics. Bold values represent the top-performing model for each metric. Values in parentheses represent the standard deviation (std). ($\downarrow$) indicates lower is better; ($\uparrow$) indicates higher is better.}
\label{table:dili}

\setlength{\tabcolsep}{6pt}

\resizebox{\textwidth}{!}{
\begin{tabular}{lcccccc}
\toprule
\textbf{Model} & \textbf{\makecell{Improvement (\%)\\($\uparrow$)}} & \textbf{\makecell{Similarity (\%)\\($\uparrow$)}} & \textbf{\makecell{Imprv $\times$ Sim\\($\uparrow$)}} & \textbf{\makecell{Drug-likeness\\($\uparrow$)}} & \textbf{\makecell{Synth\\($\downarrow$)}} & \textbf{\makecell{Validity (\%)\\($\uparrow$)}} \\
\midrule

GPT-5.4~\citep{achiam2023gpt}
& -32.053 (11.452) 
& 12.178 (18.202) 
& -390.341
& $\mathbf{0.681 \ (0.180)}$ 
& 2.729 (0.906) 
& 95.224 \\

 MolLingo (GPT-5.4 base) 
& 13.028 (7.724)
& 36.608 (15.576)
& 476.929
& 0.597 (0.189) 
& 3.087 (0.854) 
& $\mathbf{100}$ \\

Claude-4.6-Sonnet~\citep{anthropic_claude_2024}
& 9.543 (5.079) 
& 16.940 (22.726) 
& 161.658
& 0.645 (0.188) 
& 2.877 (0.871) 
& $\mathbf{100}$ \\

 MolLingo (Claude-4.6-Sonnet base)
& 10.185 (7.365) 
& 43.126 (15.705) 
& 439.238
& 0.560 (0.086) 
& 3.1851 (0.336) 
& $\mathbf{100}$ \\

Gemini-3-Pro~\citep{google_gemini_2024}
& -46.611 (13.428) 
& 14.741 (19.514) 
& -687.093
& 0.645 (0.201) 
& 2.884 (0.670) 
& 97.711 \\

\textbf{ MolLingo (Gemini-3-Pro base)}
& \textbf{18.451 (10.326)} 
& \textbf{45.515 (16.534)} 
& \textbf{839.797}
& 0.622 (0.073) 
& 3.166 (0.349) 
& \textbf{100} \\

Qwen2-7B-Instruct~\citep{qwen2} 
& 8.621 (0.395) 
& 14.866 (4.377) 
& 128.160
& 0.568 (0.033) 
& $\mathbf{2.707 \ (0.193)}$ 
& $\mathbf{100}$ \\

Qwen2-14B-Instruct~\citep{qwen2} 
& 8.621 (0.395) 
& 14.866 (4.377) 
& 128.160
& 0.568 (0.033) 
& $\mathbf{2.707 \ (0.193)}$ 
& $\mathbf{100}$ \\

\bottomrule
\end{tabular}
}
\end{table*}

\vspace{-0.5em}
\subsection{Hit-to-Lead: Context-Aware Fragment Growing}
\vspace{-0.5em}

To evaluate MolLingo's hit-to-lead pipeline, we benchmark the docking-guided fragment growing approach against the same set of frontier LLMs, now tasked with proposing fragment growths to improve the docking score of a hit molecule against a target protein. Baseline models are provided with the hit molecule's SMILES and the target protein, and are prompted to propose structurally optimized analogs. MolLingo additionally provides the Chemist Agent with the biological context of the binding site as described in \S\ref{sec:h2l}, grounding its fragment growing proposals in the three-dimensional structure of the protein--ligand binding site.

\textbf{Dataset \& Metrics.} We evaluate on a diverse set of 30 protein targets representing distinct disease areas, spanning a broad range of therapeutic indications and protein families.
This diversity ensures that results are not biased toward any particular target class or binding site geometry.
The full list of protein targets is provided in Table~\ref{tab:proteins}.
We report five metrics: (i) Validity; (ii) Docking Score Improvement (\%), the relative improvement in docking score compared to the original hit; (iii) Similarity; (iv) Drug-likeness; (v) Synth.


\textbf{Results.} Table~\ref{table:molecule_generation_summary} reports hit-to-lead results. MolLingo achieves the highest docking score improvement (10.4\%) and 100\% validity, while maintaining high structural similarity (65.2\%) to the original hits. The high similarity score indicates that MolLingo's fragment growing is targeted, expanding the hit molecule in a structurally conservative manner rather than drifting toward chemically distant analogs. In contrast, baseline LLMs operating on raw SMILES produce either modest docking improvements with low similarity (Claude-4.6-Sonnet, Gemini-3-Pro, Qwen2-family) or higher similarity but substantially lower docking improvement (GPT-5.4). The comparison between MolLingo and GPT-5.4, which share the same underlying LLM, isolates the contribution of the docking-guided biological context: GPT-5.4 alone achieves only 2.4\% docking improvement, while MolLingo achieves 10.4\%, a more than fourfold gain. Baseline drug-likeness scores are generally higher than MolLingo's, reflecting the tendency of unconstrained SMILES generation to produce drug-like but less potent molecules that do not specifically target the binding site.

\vspace{-0.5em}
\begin{table*}[!hbt]
\centering
\caption{Hit-to-lead performance comparison across frontier LLMs. MolLingo employs docking-guided biological context, while all baselines operate directly on raw SMILES + protein name.}
\label{table:molecule_generation_summary}

\setlength{\tabcolsep}{10pt}

\resizebox{\textwidth}{!}{
\begin{tabular}{lccccc}
\toprule
& \textbf{\makecell{Validity (\%)\\($\uparrow$)}} & \textbf{\makecell{Docking score \\ improvement (\%) ($\uparrow$)}} & \textbf{\makecell{Similarity (\%)\\($\uparrow$)}} & \textbf{\makecell{Drug-likeness\\($\uparrow$)}} & \textbf{\makecell{Synth\\($\downarrow$)}} \\
\midrule

\textbf{ MolLingo (GPT-5.4 base)} 
& $\mathbf{100}$ 
& $\mathbf{10.441 (4.997)}$ 
& 65.190 (3.605) 
& 0.630 (0.076)
& 3.343 (0.179) \\

GPT-5.4~\citep{achiam2023gpt}
& 98.67 
& 2.437 (2.618) 
& $\mathbf{75.871 (8.982)}$ 
& 0.769 (0.062)
& $\mathbf{2.402 (0.178)}$ \\

GPT-4.1~\citep{achiam2023gpt} 
& 97.33 
& 2.327 (2.805) 
& 73.629 (8.375)
& 0.780 (0.060)
& 2.867 (0.164) \\

Claude-4.6-Sonnet~\citep{anthropic_claude_2024} 
& $\mathbf{100}$ 
& 4.308 (4.894)
& 16.527 (1.304)
& 0.864 (0.065)
& 2.846 (0.327) \\

Gemini-3-Pro~\citep{google_gemini_2024} 
& 94 
& -1.361 (1.135)
& 14.103 (3.707)
& $\mathbf{0.878 (0.072)}$ 
& 3.104 (0.223)\\

Qwen2-14B-Instruct~\citep{qwen2} 
& $\mathbf{100}$ 
& 4.874 (2.723)
& 20.471 (4.120)
& 0.774 (0.081)
& 2.960 (0.197)\\

Qwen2-7B-Instruct~\citep{qwen2} 
& $\mathbf{100}$ 
& 4.353 (2.756)
& 20.010 (4.121)
& 0.758 (0.080)
& 2.935 (0.193)\\

\bottomrule
\end{tabular}
}
\end{table*}
\vspace{-0.5em}


\vspace{-0.4em}
\section{Related Work}
\vspace{-0.5em}
\paragraph{Agentic Multi-Agent Systems.} Recent works have explored organizing multiple specialized agents to automate distinct stages of the drug discovery pipeline. Systems such as PharmAgents~\citep{gao2025pharmagents}, Prompt-to-Pill~\citep{li2026m}, MADD~\citep{solovev2025madd}, and Tippy~\citep{fehlis2025accelerating} coordinate agents responsible for target identification, molecule generation, property evaluation, and optimization, typically following a fixed pipeline with heuristic coordination and no multi-step reasoning over molecular design objectives. More closely related are tool-augmented LLM systems such as ChemCrow~\citep{bran2023chemcrow} and Coscientist~\citep{boiko2023autonomous}, which integrate LLMs with chemistry tools and demonstrated that tool use substantially extends LLM capabilities beyond parametric knowledge. MolLingo builds on this direction but goes further: rather than a single LLM with tools, MolLingo coordinates multiple specialized agents through a shared memory module that enables iterative, evidence-driven refinement, and grounds molecular reasoning in a block-based representation designed to align with LLM semantic space.

\vspace{-0.5em}
\paragraph{Block-Based Molecular Representations.} mCLM~\citep{edwards2025mclm} introduces a modular chemical language model that tokenizes molecules into building blocks using three synthetic coupling reactions, while
FragGPT~\citep{yue2024unlocking} and Trio~\citep{ji2025toward} decompose molecules at all BRICS bonds but do not select among possible decompositions. \citep{samanta2025fragmentnet} and \citep{shen2024graphbpe} apply 
BPE at the atom level, with no synthesizability guarantee on the 
resulting fragments, while \citep{kaushal2026fragberta} selects BRICS bond cut sites based on predefined rules rather than corpus-level statistics. BFE differs in three key respects: first, it adopts the full BRICS rule set~\citep{degen2008art}; second, rather than breaking molecules at all possible bond sites as in~\citep{yue2024unlocking} and ~\citep{ji2025toward}, BFE selects the optimal decomposition based on corpus-level block frequency statistics, producing more informative and well-represented decompositions; and third, BFE requires no fine-tuned language model and is designed to leverage the parametric knowledge of any pretrained LLM through common chemical names, making it model-agnostic and immediately applicable without additional training.



\vspace{-0.4em}
\section{Conclusion, Limitations and Future Work}
\vspace{-0.5em}
We presented MolLingo, an autonomous multi-agent system that emulates the reasoning process of a chemist across the molecular design pipeline, grounding molecular reasoning in external biological evidence, physics-based simulation, and a novel block-based molecular representation that aligns with LLM semantic space. Empirical results across four benchmarks consistently demonstrate that LLMs perform dramatically better when probed in the right way. These results suggest that the bottleneck in LLM-guided molecular design lies not in the capability of the underlying model, but in how molecules, physical and biological context are communicated to it, and that LLMs, when spoken to in the language of chemistry, are already capable chemist assistants.

Despite these results, several limitations remain. The ADMET oracle models are trained on datasets of moderate size and may not generalize to highly novel compounds. Docking with AutoDock Vina approximates binding affinity without accounting for protein flexibility, solvent effects, or entropic contributions, which may limit optimization accuracy for challenging targets. MolLingo currently operates entirely in silico, and predicted improvements in docking score and ADMET properties have not yet been validated experimentally.

Future work will explore incorporating protein flexibility through molecular dynamics or ensemble docking, and expanding the ADMET oracle suite to cover additional pharmacokinetic endpoints. Most importantly, closing the loop between computational predictions and experimental validation through active learning, where wet lab outcomes inform and refine the agent's strategies, represents the most compelling direction toward a truly autonomous drug discovery system.

\subsubsection*{Acknowledgments}

We would like to thank Zhenhailong Wang and Jeonghwan Kim for helpful discussions. This work was supported by the NSF Molecule Maker Lab Institute (MMLI), an AI Institute for Molecular Discovery, Synthesis Strategy, and Manufacturing, funded by the U.S. National Science Foundation under Awards No. 2019897 and 2505932.
\clearpage
\bibliographystyle{plainnat}
\bibliography{references}

@article{jumper2021highly,
  title={Highly accurate protein structure prediction with AlphaFold},
  author={Jumper, John and Evans, Richard and Pritzel, Alexander and Green, Tim and Figurnov, Michael and Ronneberger, Olaf and Tunyasuvunakool, Kathryn and Bates, Russ and {\v{Z}}{\'\i}dek, Augustin and Potapenko, Anna and others},
  journal={nature},
  volume={596},
  number={7873},
  pages={583--589},
  year={2021},
  publisher={Nature Publishing Group}
}

@inproceedings{lewis2020retrieval,
 author = {Patrick S. H. Lewis and
Ethan Perez and
Aleksandra Piktus and
Fabio Petroni and
Vladimir Karpukhin and
Naman Goyal and
Heinrich K{\"{u}}ttler and
Mike Lewis and
Wen{-}tau Yih and
Tim Rockt{\"{a}}schel and
Sebastian Riedel and
Douwe Kiela},
 bibsource = {dblp computer science bibliography, https://dblp.org},
 booktitle = {Advances in Neural Information Processing Systems 33: Annual Conference
on Neural Information Processing Systems 2020, NeurIPS 2020, December
6-12, 2020, virtual},
 editor = {Hugo Larochelle and
Marc'Aurelio Ranzato and
Raia Hadsell and
Maria{-}Florina Balcan and
Hsuan{-}Tien Lin},
 title = {Retrieval-Augmented Generation for Knowledge-Intensive {NLP} Tasks},
 url = {https://proceedings.neurips.cc/paper/2020/hash/6b493230205f780e1bc26945df7481e5-Abstract.html},
 year = {2020}
}

@article{tanimoto1958elementary,
  title={Elementary mathematical theory of classification and prediction},
  author={Tanimoto, Taffee T},
  year={1958},
  publisher={International Business Machines Corp.}
}

@article{trott2010autodock,
  title={AutoDock Vina: improving the speed and accuracy of docking with a new scoring function, efficient optimization, and multithreading},
  author={Trott, Oleg and Olson, Arthur J},
  journal={Journal of computational chemistry},
  volume={31},
  number={2},
  pages={455--461},
  year={2010},
  publisher={Wiley Online Library}
}

@article{bran2023chemcrow,
  title={ChemCrow: Augmenting large-language models with chemistry tools},
  author={Bran, Andres M and Cox, Sam and White, Andrew D and Schwaller, Philippe},
  journal={arXiv preprint arXiv:2304.05376},
  year={2023}
}

@article{degen2008art,
 author = {Degen, Jorg and Wegscheid-Gerlach, Christof and Zaliani, Andrea and Rarey, Matthias},
 journal = {ChemMedChem},
 number = {10},
 pages = {1503},
 title = {On the art of compiling and using'drug-like'chemical fragment spaces},
 volume = {3},
 year = {2008}
}

@article{bemis1996properties,
  title={The properties of known drugs. 1. Molecular frameworks},
  author={Bemis, Guy W and Murcko, Mark A},
  journal={Journal of medicinal chemistry},
  volume={39},
  number={15},
  pages={2887--2893},
  year={1996},
  publisher={ACS Publications}
}

@article{ertl2009estimation,
  title={Estimation of synthetic accessibility score of drug-like molecules based on molecular complexity and fragment contributions},
  author={Ertl, Peter and Schuffenhauer, Ansgar},
  journal={Journal of cheminformatics},
  volume={1},
  number={1},
  pages={8},
  year={2009},
  publisher={Springer}
}

@article{sadybekov2023computational,
  title={Computational approaches streamlining drug discovery},
  author={Sadybekov, Anastasiia V and Katritch, Vsevolod},
  journal={Nature},
  volume={616},
  number={7958},
  pages={673--685},
  year={2023},
  publisher={Nature Publishing Group UK London}
}

@article{achiam2023gpt,
  title={Gpt-4 technical report},
  author={Achiam, Josh and Adler, Steven and Agarwal, Sandhini and Ahmad, Lama and Akkaya, Ilge and Aleman, Florencia Leoni and Almeida, Diogo and Altenschmidt, Janko and Altman, Sam and Anadkat, Shyamal and others},
  journal={arXiv preprint arXiv:2303.08774},
  year={2023}
}

@article{boiko2023autonomous,
  title={Autonomous chemical research with large language models},
  author={Boiko, Daniil A and MacKnight, Robert and Kline, Ben and Gomes, Gabe},
  journal={Nature},
  volume={624},
  number={7992},
  pages={570--578},
  year={2023},
  publisher={Nature Publishing Group UK London}
}

@article{weininger1988smiles,
  title={SMILES, a chemical language and information system. 1. Introduction to methodology and encoding rules},
  author={Weininger, David},
  journal={Journal of chemical information and computer sciences},
  volume={28},
  number={1},
  pages={31--36},
  year={1988},
  publisher={ACS Publications}
}

@article{guo2023can,
  title={What can large language models do in chemistry? a comprehensive benchmark on eight tasks},
  author={Guo, Taicheng and Nan, Bozhao and Liang, Zhenwen and Guo, Zhichun and Chawla, Nitesh and Wiest, Olaf and Zhang, Xiangliang and others},
  journal={Advances in neural information processing systems},
  volume={36},
  pages={59662--59688},
  year={2023}
}

@article{ucak2023reconstruction,
  title={Reconstruction of lossless molecular representations from fingerprints},
  author={Ucak, Umit V and Ashyrmamatov, Islambek and Lee, Juyong},
  journal={Journal of cheminformatics},
  volume={15},
  number={1},
  pages={26},
  year={2023},
  publisher={Springer}
}

@article{sutherland2008chemical,
  title={Chemical fragments as foundations for understanding target space and activity prediction},
  author={Sutherland, Jeffrey J and Higgs, Richard E and Watson, Ian and Vieth, Michal},
  journal={Journal of medicinal chemistry},
  volume={51},
  number={9},
  pages={2689--2700},
  year={2008},
  publisher={ACS Publications}
}

@article{edwards2025mclm,
  title={mclm: A function-infused and synthesis-friendly modular chemical language model},
  author={Edwards, Carl and Han, Chi and Lee, Gawon and Nguyen, Thao and Jin, Bowen and Prasad, Chetan Kumar and Szymkuc, Sara and Grzybowski, Bartosz A and Diao, Ying and Han, Jiawei and others},
  journal={arXiv preprint arXiv:2505.12565},
  year={2025}
}

@article{unlu2025target,
  title={Target-specific de novo design of drug candidate molecules with graph-transformer-based generative adversarial networks},
  author={{\"U}nl{\"u}, Atabey and {\c{C}}evrim, Elif and Yi{\u{g}}it, Melih G{\"o}kay and Sar{\i}g{\"u}n, Ahmet and {\c{C}}elikbilek, Hayriye and Bayram, Osman and Kahraman, Deniz Cansen and Ol{\u{g}}a{\c{c}}, Abdurrahman and Rifaioglu, Ahmet Sureyya and Bano{\u{g}}lu, Erden and others},
  journal={Nature Machine Intelligence},
  volume={7},
  number={9},
  pages={1524--1540},
  year={2025},
  publisher={Nature Publishing Group UK London}
}

@article{gaulton2012chembl,
  title={ChEMBL: a large-scale bioactivity database for drug discovery},
  author={Gaulton, Anna and Bellis, Louisa J and Bento, A Patricia and Chambers, Jon and Davies, Mark and Hersey, Anne and Light, Yvonne and McGlinchey, Shaun and Michalovich, David and Al-Lazikani, Bissan and others},
  journal={Nucleic acids research},
  volume={40},
  number={D1},
  pages={D1100--D1107},
  year={2012},
}

@article{gao2023drugclip,
  title={Drugclip: Contrastive protein-molecule representation learning for virtual screening},
  author={Gao, Bowen and Qiang, Bo and Tan, Haichuan and Jia, Yinjun and Ren, Minsi and Lu, Minsi and Liu, Jingjing and Ma, Wei-Ying and Lan, Yanyan},
  journal={Advances in Neural Information Processing Systems},
  volume={36},
  pages={44595--44614},
  year={2023}
}

@article{nguyen2024farm,
  title={Farm: Functional group-aware representations for small molecules},
  author={Nguyen, Thao and Huang, Kuan-Hao and Liu, Ge and Burke, Martin D and Diao, Ying and Ji, Heng},
  journal={arXiv preprint arXiv:2410.02082},
  year={2024}
}

@article{qwen2,
  title   = {Qwen2 Technical Report},
  author  = {Team Qwen},
  journal = {arXiv preprint arXiv:2407.10671},
  year    = {2024},
  url     = {https://arxiv.org/abs/2407.10671}
}

@article{bickerton2012quantifying,
  title={Quantifying the chemical beauty of drugs},
  author={Bickerton, G Richard and Paolini, Gaia V and Besnard, J{\'e}r{\'e}my and Muresan, Sorel and Hopkins, Andrew L},
  journal={Nature chemistry},
  volume={4},
  number={2},
  pages={90--98},
  year={2012},
  publisher={Nature Publishing Group UK London}
}

@article{bian2018computational,
  title={Computational fragment-based drug design: current trends, strategies, and applications},
  author={Bian, Yuemin and Xie, Xiang-Qun},
  journal={The AAPS journal},
  volume={20},
  number={3},
  pages={59},
  year={2018},
  publisher={Springer}
}

@article{congreve2003rule,
  title={A'rule of three'for fragment-based lead discovery?},
  author={Congreve, Miles and Carr, Robin and Murray, Chris and Jhoti, Harren},
  journal={Drug discovery today},
  volume={8},
  number={19},
  pages={876--877},
  year={2003}
}

@article{huang2021therapeutics,
  title={Therapeutics data commons: Machine learning datasets and tasks for drug discovery and development},
  author={Huang, Kexin and Fu, Tianfan and Gao, Wenhao and Zhao, Yue and Roohani, Yusuf and Leskovec, Jure and Coley, Connor W and Xiao, Cao and Sun, Jimeng and Zitnik, Marinka},
  journal={arXiv preprint arXiv:2102.09548},
  year={2021}
}

@article{gallo2024withdrawn,
  title={Withdrawn 2.0—update on withdrawn drugs with pharmacovigilance data},
  author={Gallo, Kathleen and Goede, Andrean and Eckert, Oliver-Andreas and Gohlke, Bjoern-Oliver and Preissner, Robert},
  journal={Nucleic Acids Research},
  volume={52},
  number={D1},
  pages={D1503--D1507},
  year={2024},
  publisher={Oxford University Press}
}

@article{li2026reference,
  title={Reference-guided Policy Optimization for Molecular Optimization via LLM Reasoning},
  author={Li, Xuan and Zhou, Zhanke and Li, Zongze and Yao, Jiangchao and Rong, Yu and Zhang, Lu and Han, Bo},
  journal={arXiv preprint arXiv:2603.05900},
  year={2026}
}

@article{wang2018remol,
  title={ReMol: LLM-guided Molecular Optimization with Reinforcement Learning},
  author={Wang, Ziqing and Ding, Kaize},
  year={2018}
}

@article{li2024tomg,
  title={Tomg-bench: Evaluating llms on text-based open molecule generation},
  author={Li, Jiatong and Li, Junxian and Liu, Yunqing and Zhou, Dongzhan and Li, Qing},
  journal={arXiv preprint arXiv:2412.14642},
  year={2024},
  publisher={Dec}
}

@article{berman2000protein,
  title={The protein data bank},
  author={Berman, Helen M and Westbrook, John and Feng, Zukang and Gilliland, Gary and Bhat, Talapady N and Weissig, Helge and Shindyalov, Ilya N and Bourne, Philip E},
  journal={Nucleic acids research},
  volume={28},
  number={1},
  pages={235--242},
  year={2000},
  publisher={Oxford University Press}
}

@article{butina1999unsupervised,
  title={Unsupervised data base clustering based on daylight's fingerprint and Tanimoto similarity: A fast and automated way to cluster small and large data sets},
  author={Butina, Darko},
  journal={Journal of Chemical Information and Computer Sciences},
  volume={39},
  number={4},
  pages={747--750},
  year={1999},
  publisher={ACS Publications}
}

@article{lynch2004activating,
  title={Activating mutations in the epidermal growth factor receptor underlying responsiveness of non--small-cell lung cancer to gefitinib},
  author={Lynch, Thomas J and Bell, Daphne W and Sordella, Raffaella and Gurubhagavatula, Sarada and Okimoto, Ross A and Brannigan, Brian W and Harris, Patricia L and Haserlat, Sara M and Supko, Jeffrey G and Haluska, Frank G and others},
  journal={New England Journal of Medicine},
  volume={350},
  number={21},
  pages={2129--2139},
  year={2004},
  publisher={Mass Medical Soc}
}

@article{gao2025pharmagents,
  title={Pharmagents: Building a virtual pharma with large language model agents},
  author={Gao, Bowen and Huang, Yanwen and Liu, Yiqiao and Xie, Wenxuan and Ma, Wei-Ying and Zhang, Ya-Qin and Lan, Yanyan},
  journal={arXiv preprint arXiv:2503.22164},
  year={2025}
}

@article{li2026m,
  title={M\^{} 4olGen: Multi-Agent, Multi-Stage Molecular Generation under Precise Multi-Property Constraints},
  author={Li, Yizhan and Cloutier, Florence and Wu, Sifan and Parviz, Ali and Knyazev, Boris and Zhang, Yan and Berseth, Glen and Liu, Bang},
  journal={arXiv preprint arXiv:2601.10131},
  year={2026}
}

@inproceedings{solovev2025madd,
  title={MADD: Multi-Agent Drug Discovery Orchestra},
  author={Solovev, Gleb V and Zhidkovskaya, Alina B and Orlova, Anastasia and Gubina, Nina and Vepreva, Anastasia and Golovinskii, Rodion and Tonkii, Ilya and Dubrovsky, Ivan and Gurev, Ivan and Gilemkhanov, Dmitry and others},
  booktitle={Findings of the Association for Computational Linguistics: EMNLP 2025},
  pages={6956--6998},
  year={2025}
}

@article{fehlis2025accelerating,
  title={Accelerating drug discovery through agentic ai: A multi-agent approach to laboratory automation in the dmta cycle},
  author={Fehlis, Yao and Crain, Charles and Jensen, Aidan and Watson, Michael and Juhasz, James and Mandel, Paul and Liu, Betty and Mahon, Shawn and Wilson, Daren and Lynch-Jonely, Nick and others},
  journal={arXiv preprint arXiv:2507.09023},
  year={2025}
}

@techreport{anthropic_claude_2024,
  author      = {Anthropic},
  title       = {Claude 3 Model Card},
  institution = {Anthropic},
  year        = {2024},
  url         = {https://www.anthropic.com}
}

@techreport{google_gemini_2024,
  author      = {Google DeepMind},
  title       = {Gemini Model Technical Report / Model Card},
  institution = {Google DeepMind},
  year        = {2024},
  url         = {https://deepmind.google}
}

@article{bondi1964van,
  title={van der Waals Volumes and Radii},
  author={Bondi, A van},
  journal={The Journal of physical chemistry},
  volume={68},
  number={3},
  pages={441--451},
  year={1964},
  publisher={ACS Publications}
}

@inproceedings{yao2022react,
  title={React: Synergizing reasoning and acting in language models},
  author={Yao, Shunyu and Zhao, Jeffrey and Yu, Dian and Du, Nan and Shafran, Izhak and Narasimhan, Karthik R and Cao, Yuan},
  booktitle={The eleventh international conference on learning representations},
  year={2022}
}

@article{yue2024unlocking,
  title={Unlocking comprehensive molecular design across all scenarios with large language model and unordered chemical language},
  author={Yue, Jie and Peng, Bingxin and Chen, Yu and Jin, Jieyu and Zhao, Xinda and Shen, Chao and Ji, Xiangyang and Hsieh, Chang-Yu and Song, Jianfei and Hou, Tingjun and others},
  journal={Chemical Science},
  volume={15},
  number={34},
  pages={13727--13740},
  year={2024},
  publisher={Royal Society of Chemistry}
}

@article{ji2025toward,
  title={Toward Closed-loop Molecular Discovery via Language Model, Property Alignment and Strategic Search},
  author={Ji, Junkai and Yang, Zhangfan and Xu, Dong and Bai, Ruibin and Li, Jianqiang and Hou, Tingjun and Zhu, Zexuan},
  journal={arXiv preprint arXiv:2512.09566},
  year={2025}
}

@article{kaushal2026fragberta,
  title={FragBERTa: Exploring Fragment-based Molecular Representation Learning with SAFE},
  author={Kaushal, Neerav and Penmatsa, Ajay MNV},
  year={2026}
}

@article{samanta2025fragmentnet,
  title={FragmentNet: Adaptive Graph Fragmentation for Graph-to-Sequence Molecular Representation Learning},
  author={Samanta, Ankur and Gupta, Rohan and Misra, Aditi and Clarke, Christian McIntosh and Rajadas, Jayakumar},
  journal={arXiv preprint arXiv:2502.01184},
  year={2025}
}

@article{shen2024graphbpe,
  title={Graphbpe: Molecular graphs meet byte-pair encoding},
  author={Shen, Yuchen and P{\'o}czos, Barnab{\'a}s},
  journal={arXiv preprint arXiv:2407.19039},
  year={2024}
}

\clearpage
\appendix



\section{Background: Early-Stage Drug Discovery Pipeline}
\label{app:dd-pipeline}
Drug discovery is a lengthy and costly process spanning from initial biological hypothesis to clinical approval. The early computational stages, often referred to as the dry lab phase, are critical for narrowing down the vast chemical space to a small set of promising candidates before any wet lab experimentation begins. As illustrated in Fig.~\ref{fig:dd_pipeline}, this phase proceeds through five sequential stages. Target identification establishes the biological rationale by linking a disease phenotype to a specific protein or pathway whose modulation may yield therapeutic benefit. Hit identification screens large chemical libraries to find initial compounds (referred to as hits) that bind to the target with detectable affinity. Hit-to-lead progression elaborates the most promising hits through structural modifications to improve potency and selectivity while maintaining drug-like properties. Lead optimization further refines the lead compound to achieve a favorable balance of efficacy, selectivity, and ADMET properties. Finally, candidate selection identifies the molecule best suited for advancement into preclinical and clinical studies, based on its overall profile across all evaluated criteria. MolLingo is designed to automate and accelerate each of these computational stages through LLM-guided reasoning and multi-agent coordination, as described in the main text.

\begin{figure}[!hbt]
 \centering
 \includegraphics[width=\linewidth]{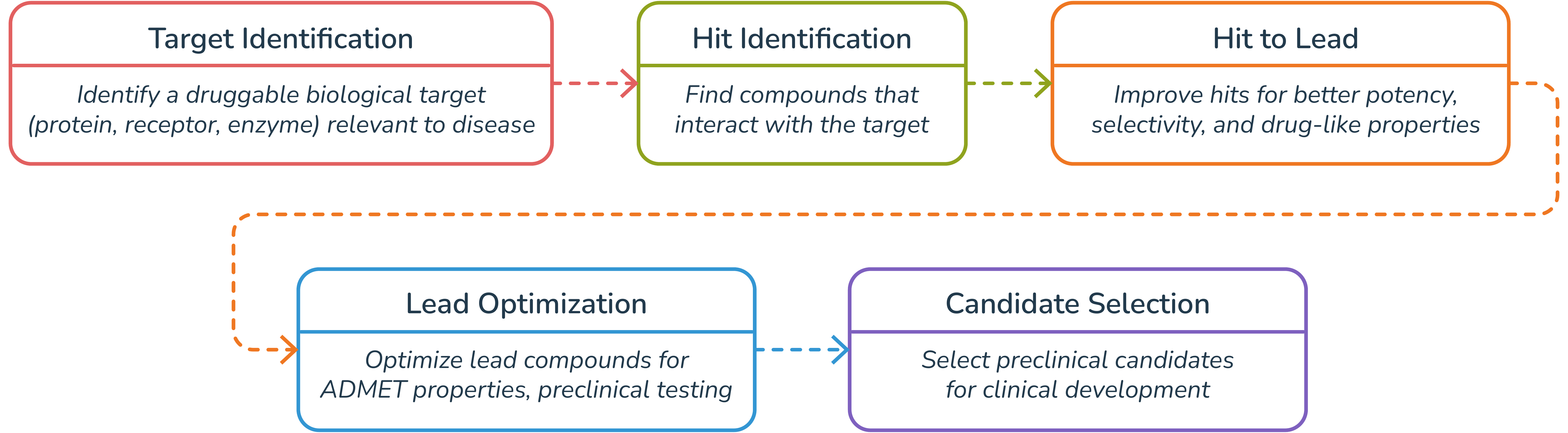}
 \caption{Overview of the early-stage computational drug discovery pipeline, covering the dry lab phases from target identification through candidate selection. Each stage progressively narrows the chemical space from thousands of initial hits to a single optimized drug candidate, guided by increasingly stringent criteria for potency, selectivity, and drug-likeness. MolLingo automates this end-to-end process through coordinated multi-agent reasoning, operating entirely in silico prior to experimental validation.}
 \label{fig:dd_pipeline}
\end{figure}

\section{BRICS-Based Fragment Enumeration (BFE)}
\label{sec:bfe}
\subsection*{Vocabulary Construction}

We train the tokenizer over a corpus of 20M drug-like molecules drawn from ChEMBL, ZINC, and industrial fragment and screening libraries introduced in~\cite{nguyen2024farm}. For each molecule, we first apply BRICS~\citep{degen2008art} to identify all synthesizable breakable bonds (single bonds whose breaking yields chemically meaningful, synthetically accessible fragments). The resulting fragments serve as the primitive units of our vocabulary — analogous to characters in text BPE — with the key distinction that each primitive is a chemically meaningful, synthetically accessible fragment rather than an arbitrary character. This ensures that all building blocks in the vocabulary correspond to realistically synthesizable units, a property that character-level fragmentation cannot guarantee.

\subsection*{The BFE Algorithm}
Unlike text BPE, which operates on linear character sequences, molecular graphs present two key challenges for direct BPE application: (i) fragments can have more than two neighbors, making it impossible to represent branched molecules as a flat sequence of blocks without auxiliary notation; and (ii) merging two molecular graphs programmatically is expensive (see Proposition~\ref{prop1}), making iterative graph-based vocabulary construction costly at corpus scale.

We resolve both challenges by inverting the BPE construction logic: rather than iteratively merging fragments, we enumerate building blocks by exhaustively breaking each molecule at all possible pairs of its BRICS bonds using \texttt{FindBRICSBonds}. By Lemma~\ref{lemma1}, breaking at exactly two bonds (treating terminal boundaries as virtual bonds) is both necessary and sufficient to recover all valid blocks. 
For example, for a molecule with sequential blocks $A$-$B$-$C$-$D$, this yields all valid fragments $\{A, B, C, D, AB, BC, CD, ABC, BCD\}$. Bond pairs that would produce branched decompositions are discarded. The frequency of each resulting fragment is accumulated in a single pass over the full corpus, producing a vocabulary of building blocks ranked by their prevalence in drug-like chemical space, at a total cost of $O(|\mathcal{C}| \cdot \bar{n}^3)$, which is faster than iterative graph-based vocabulary construction by a factor of $O\big((V - n_0)\,\bar{n}\big)$ (where $V$ is the target vocabulary size, $n_0$ is the initial atomic vocabulary size, and $\bar{n}$ is the average number of blocks per molecule) as formally derived in Corollary~\ref{cor:corpus}. The pseudocode for vocabulary construction is provided in Algorithm~\ref{alg:vocab}.

\begin{lemma}[Enumeration recovers all possible blocks]
\label{lemma1}
Under the no-branch assumption, the molecule $M$ is a path graph. Any 
contiguous sub-sequence $[b_i, \ldots, b_j]$ is a valid block and can 
be obtained by breaking exactly two bonds, where terminal boundaries of 
the molecule are treated as virtual bonds.
\begin{proof}
Let $M = [b_1, b_2, \ldots, b_n]$. Augment the bond set by introducing 
two virtual terminal bonds: a left boundary bond $(b_0, b_1)$ and a 
right boundary bond $(b_n, b_{n+1})$, where $b_0$ and $b_{n+1}$ are 
sentinel nodes. Define the augmented bond set:
\[
\tilde{B} = \{(b_0, b_1)\} \cup \{(b_k, b_{k+1})\}_{k=1}^{n-1} 
\cup \{(b_n, b_{n+1})\}
\]
Then any contiguous sub-sequence $[b_i, \ldots, b_j]$ is recovered by 
breaking exactly two bonds from $\tilde{B}$:
\[
S = \{(b_{i-1}, b_i),\ (b_j, b_{j+1})\}, \quad |S| = 2 \quad 
\forall\, 1 \leq i \leq j \leq n.
\]
Since every choice of $1 \leq i \leq j \leq n$ yields a distinct 
contiguous sub-sequence via exactly this bond pair, the enumeration 
over all $S \subseteq \tilde{B}$ with $|S| = 2$ is complete.
\end{proof}
\end{lemma}







\begin{proposition}[Per-molecule computational cost]
\label{prop1}
Let $M = [b_1, \ldots, b_n]$ be a molecule with $n$ blocks. Let $C_{\text{break}} = O(n)$ denote the cost of one RDKit bond-breaking operation, and $C_{\text{merge}} = O(n^2)$ denote the cost of one RDKit fragment-merging operation, where the $O(n^2)$ term arises from sanitization (valence checking, aromaticity perception, and ring detection).

\begin{enumerate}[label=(\roman*)]
 \item \textbf{BFE:}
 \[
 T_{\text{BFE}} = \binom{n+1}{2} \cdot C_{\text{break}}
 = \frac{n(n+1)}{2} \cdot O(n)
 = O(n^3).
 \]

 \item \textbf{Graph BPE:}
 \[
 T_{\text{GraphBPE}} = \frac{n(n-1)}{2} \cdot C_{\text{merge}}
 = \frac{n(n-1)}{2} \cdot O(n^2)
 = O(n^4).
 \]
\end{enumerate}

\begin{proof}

\textbf{(i) BFE.} By Lemma 1, each valid block corresponds to a unique pair of bonds in $\tilde{B}$ with $|\tilde{B}| = n+1$. Hence, the number of breaking operations is $\binom{n+1}{2}$. Each breaking action invokes \texttt{FragmentOnBonds} on the full molecule of $n$ atoms, which requires copying the molecular graph and removing an edge, both $O(n)$, giving $C_{\text{break}} = O(n)$. Thus:
\[
T_{\text{BFE}} = \binom{n+1}{2} \cdot O(n) = O(n^3).
\]

\textbf{(ii) Graph BPE.} Starting from $n$ blocks, at each merge step, the most frequent adjacent pair across the corpus is selected and collapsed into a single new block, reducing the block count by one. For a single molecule, this process repeats until all blocks are merged, yielding $\sum_{k=1}^{n-1} k = \frac{n(n-1)}{2}$ merge operations in the worst case. Each operation involves \texttt{CombineMols}, which scales linearly in the number of atoms and bonds, followed by bond formation and mandatory sanitization. Assuming sanitization dominates (due to, e.g., ring perception) we take $C_{\text{merge}} = O(n^2).$ Therefore:
\[
T_{\text{GraphBPE}} = \frac{n(n-1)}{2} \cdot O(n^2) = O(n^4).
\]
\end{proof}
\end{proposition}

\begin{corollary}[Per-molecule speedup]
\[
\frac{T_{\text{GraphBPE}}}{T_{\text{BFE}}} 
= \frac{\dfrac{n(n-1)}{2} \cdot O(n^2)}{\dfrac{n(n+1)}{2} \cdot O(n)}
= \frac{n-1}{n+1} \cdot O(n)
= O(n)
\]
BFE achieves an $O(n)$ per-molecule speedup over Graph BPE, owing to 
the cheaper $O(n)$ cost of \texttt{FragmentOnBonds} compared to the 
$O(n^2)$ cost of merge-and-sanitize. The efficiency advantage further 
compounds at the corpus level, as described in Corollary~\ref{cor:corpus}.
\end{corollary}

\begin{corollary}[Corpus-level speedup]
\label{cor:corpus}
Over a corpus $\mathcal{C}$ of molecules with average block count 
$\bar{n}$, Graph BPE requires $V - n_0$ full corpus passes to reach 
a target vocabulary size $V$ from an initial atomic vocabulary of size 
$n_0$, while BFE requires exactly one pass. The total corpus-level costs are:
\[
T_{\text{BFE}}^{\text{corpus}} = |\mathcal{C}| \cdot O(\bar{n}^3)
\]
\[
T_{\text{GraphBPE}}^{\text{corpus}} = (V - n_0) \cdot |\mathcal{C}| 
\cdot O(\bar{n}^4)
\]
yielding a corpus-level speedup that scales as:
\[
\boxed{O\big((V - n_0)\,\bar{n}\big)}
\]
\end{corollary}

We empirically evaluate the scaling behavior of the merge-to-break time ratio as a function of molecule size. For molecules containing 10, 15, and 20 heavy atoms, the observed ratios were $11.84\times$, $15.22\times$, and $20.35\times$, respectively, averaged over 300 molecules per size. These values increase approximately linearly with the number of atoms $n$, closely tracking the expected trend. This result provides experimental confirmation that the merge-to-break cost scales linearly as $O(n)$.

For example, for values ($V = 2000$, $n_0 = 100$, $\bar{n} = 8$), the approximate speedup is:
$\text{Speedup} \approx 1900 \times 8
= 15200$. This suggests an approximate speedup on the order of $10^4\times$, though the exact factor depends on implementation constants.

\begin{table*}[!t]
\centering
\caption{Computational cost comparison between BFE vocabulary construction and Graph BPE. $n$ denotes the number of blocks per molecule, $|\mathcal{C}|$ the corpus size, $V$ the target vocabulary size, and $n_0$ the initial atomic vocabulary size.}
\label{tab:complexity}

\setlength{\tabcolsep}{15pt}

\resizebox{\textwidth}{!}{
\begin{tabular}{lccccc}
\toprule
& \textbf{\# Actions} & \textbf{Action cost} & 
\textbf{Per-molecule} & \textbf{Corpus passes} & \textbf{Total} \\
\midrule

\textbf{BFE} 
& $\dfrac{n(n+1)}{2}$ 
& $O(n)$ 
& $O(n^3)$ 
& $1$ 
& $O(|\mathcal{C}| \cdot \bar{n}^3)$ \\[10pt]

\textbf{Graph BPE} 
& $\dfrac{n(n-1)}{2}$ 
& $O(n^2)$ 
& $O(n^4)$ 
& $V - n_0$ 
& $O((V-n_0)\cdot|\mathcal{C}|\cdot\bar{n}^4)$ \\[10pt]

\textbf{Speedup} 
& $\dfrac{n-1}{n+1}$ 
& $O(n)$ 
& $O(n)$ 
& $V - n_0$ 
& $(V-n_0)\cdot O(\bar{n})$ \\
\bottomrule
\end{tabular}
}
\end{table*}

\begin{algorithm}[!hbt]
\caption{Phase 1 — Vocabulary Construction}
\label{alg:vocab}
\begin{algorithmic}[1]
\Require Corpus $\mathcal{C}$ of 20M molecules (ChEMBL, ZINC, screening and fragment libraries)
\Ensure Fragment frequency dictionary $\texttt{vocab}$

\State $\texttt{vocab} \gets \emptyset$

\For{each molecule $mol \in \mathcal{C}$}
 \State $\tilde{B} \gets \textsc{BRICS\_bonds}(mol) + \text{virtual terminal bonds}$
 \State $\texttt{bond\_sets} \gets \{ S \subseteq \tilde{B} \mid|S| = 2 \}$

 \For{each $S \in \texttt{bond\_sets}$}
 \State $\texttt{fragments} \gets \textsc{break\_molecule}(mol, S)$

 \If{$\textsc{has\_branch}(\texttt{fragments})$}
 \State \textbf{continue}
 \EndIf

 \For{\texttt{block} $\in \texttt{fragments}$}
 \State $\texttt{vocab}[\texttt{block}] \gets \texttt{vocab}[\texttt{block}] + 1$
 \EndFor
 \EndFor
\EndFor

\State \Return $\texttt{vocab}$
\end{algorithmic}
\end{algorithm}

\subsection*{Molecular Fragmentation}
Given a molecule, we decompose it into a sequence of building blocks 
from the learned vocabulary using \texttt{FragmentOnBonds}. For each 
molecule, we enumerate all valid non-branching decompositions by 
breaking all possible pairs of BRICS bonds, as in vocabulary 
construction. Each resulting decomposition is a candidate fragmentation: a sequence of blocks $[t_1, t_2, \ldots, t_k]$ where each $t_i$ 
belongs to the vocabulary.

Among all candidate decompositions, we select the optimal one hierarchically. We first prioritize decompositions with the fewest fragments (i.e., longest blocks), as longer blocks carry richer structural and functional information. Among candidates of equal length (i.e., number of blocks), we retain only those whose blocks all appear in the vocabulary with frequency above a minimum threshold $f_{\min}$. If multiple 
candidates satisfy this condition, we select the one with the lowest standard deviation of block frequencies:
\[
t^* = \underset{t \in \mathcal{T}^*}{\arg\min}\ 
\text{std}\left(\{freq(t_i)\}_{i=1}^{k}\right)
\]
where $\mathcal{T}^*$ is the set of valid candidates with the fewest fragments and all block frequencies above $f_{\min}$. In our experiments, $f_{\min}$ is set to 20, yielding a final vocabulary of approximately 1.2M building blocks. This criterion favors decompositions whose blocks are uniformly well-represented in the vocabulary, avoiding decompositions that rely on rare or potentially noisy fragments. The pseudocode for molecular fragmentation is provided in Algorithm~\ref{alg:tokenization}.

\begin{algorithm}[!hbt]
\caption{Phase 2 — Molecular Fragmentation}
\label{alg:tokenization}
\begin{algorithmic}[1]
\Require Molecule $mol$, vocabulary $\texttt{vocab}$, threshold $\texttt{min\_freq}$
\Ensure Sequence of building blocks $[t_1, t_2, \dots, t_k]$

\State $\tilde{B} \gets \textsc{BRICS\_bonds}(mol) + \text{virtual terminal bonds}$
\State $\texttt{candidates} \gets [\ ]$

\Comment{Enumerate all valid (non-branching) decompositions}
\For{each $S \subseteq \tilde{B}$ such that $|S| = 2$}
 \State $\texttt{fragments} \gets \textsc{break\_molecule}(mol, S)$
 \If{\textbf{not} $\textsc{has\_branch}(\texttt{fragments})$}
 \State $\texttt{candidates}.\textsc{append}(\texttt{fragments})$
 \EndIf
\EndFor

\Comment{Prefer fewer fragments (longer blocks)}
\State $\textsc{sort}(\texttt{candidates}, \text{key} = \textsc{len}, \text{order} = \textsc{ascending})$

\For{$\texttt{fragmentation} \in \texttt{candidates}$}
 \State $\texttt{freqs} \gets [\texttt{vocab}[t] \;\textbf{for each}\; t \in \texttt{fragmentation}]$

 \If{$\forall f \in \texttt{freqs},\; f \ge \texttt{min\_freq}$}
 \State $\texttt{valid} \gets \{c \in \texttt{candidates} \mid |c| = |\texttt{fragmentation}|\}$
 \State \Return $\arg\min_{c \in \texttt{valid}} \textsc{std}(\texttt{vocab}[t], t \in c)$
 \EndIf
\EndFor

\Comment{Fallback: finest-grained decomposition}
\State \Return $\arg\min_{c \in \texttt{candidates}} \textsc{len}(c)$

\end{algorithmic}
\end{algorithm}

\subsection*{BFE vs. BRICS Fragmentation}

To isolate the contribution of BFE's frequency-guided optimal 
decomposition selection from a naive BRICS decomposition baseline, 
we compare MolLingo using BFE fragmentation against a variant that 
simply breaks molecules at all possible BRICS bonds without 
vocabulary-guided selection (denoted BRICS). Both variants use 
Gemini-3-Pro as the base LLM and are evaluated on the DILI 
benchmark.

\begin{table}[h]
\centering
\caption{Ablation comparing BFE fragmentation against naive BRICS 
fragmentation, using Gemini-3-Pro as the base LLM on the DILI 
benchmark. Values in parentheses denote standard deviation.}
\label{tab:bfe-vs-brics}
\setlength{\tabcolsep}{5pt}
\resizebox{\columnwidth}{!}{
\begin{tabular}{lcccccc}
\toprule
\textbf{Model} & \textbf{\makecell{Improvement\\(\%)~($\uparrow$)}} & 
\textbf{\makecell{Similarity\\(\%)~($\uparrow$)}} & 
\textbf{\makecell{Imprv$\times$Sim\\($\uparrow$)}} & 
\textbf{\makecell{Drug-likeness\\($\uparrow$)}} & 
\textbf{\makecell{Synth\\($\downarrow$)}} & 
\textbf{\makecell{Validity\\(\%)~($\uparrow$)}} \\
\midrule
\textbf{ MolLingo (Gemini-3-Pro)} 
& \textbf{18.451 (10.326)} 
& \textbf{45.515 (16.534)} 
& \textbf{839.797} 
& 0.622 (0.073) 
& 3.166 (0.349) 
& \textbf{100} \\
BRICS (Gemini-3-Pro) 
& 1.034 (8.274) 
& 78.650 (15.761) 
& 81.316 
& \textbf{0.642 (0.058)} 
& \textbf{2.934 (0.341)} 
& \textbf{100} \\
\bottomrule
\end{tabular}
}
\end{table}

The results highlight a clear tradeoff. The naive BRICS baseline 
achieves higher similarity (78.7\% vs. 45.5\%), as breaking at 
all possible bonds produces finer-grained, more conservative 
decompositions. However, this comes at a significant cost to 
optimization effectiveness: BRICS achieves only 1.0\% improvement 
compared to 18.5\% for BFE, and the Imprv$\times$Sim score 
collapses from 839.8 to 81.3. This suggests that naive BRICS 
decomposition fragments molecules into units that are too small 
and too numerous for the LLM to reason about meaningfully, losing 
the structural and semantic coherence that BFE's vocabulary-guided 
fragmentation preserves. BFE's optimal selection of longer, 
frequency-validated blocks is therefore critical to enabling 
effective LLM-guided molecular optimization.

\section{Fragment Library for Hit Identification}
\label{sec:fragment-library}

\subsection*{Fragment-Based Drug Design}

Fragment-based drug design (FBDD) has emerged as a powerful paradigm for early-stage drug discovery, complementing high-throughput screening by exploring chemical space more efficiently with smaller, simpler molecules~\citep{bian2018computational}. Rather than screening large, complex compounds, FBDD begins with low-molecular-weight fragments (typically $<$ 300 Da) that bind weakly but efficiently to the target, and progressively elaborates them into drug-like leads through fragment growing, merging, or linking~\citep{bian2018computational}.

The quality of a fragment library is critical to the success of FBDD. Fragments should satisfy the ``rule of three''~\citep{congreve2003rule} (molecular weight $\leq$ 300 Da, cLogP $\leq$ 3, hydrogen bond donors $\leq$ 3, hydrogen bond acceptors $\leq$ 3), exhibit favorable ADMET properties to ensure downstream drug-likeness, and provide broad coverage of diverse chemical scaffolds to maximize the probability of identifying a hit against a novel target. Scaffold diversity is particularly important: a library enriched with structurally redundant fragments offers little advantage over a smaller one, as redundant fragments tend to bind the same pockets in the same way.

\subsection*{ MolLingo Fragment Library}

We construct a curated fragment library by aggregating commercially available fragment collections from ChemDiv\footnote{\url{https://www.chemdiv.com/}}, Enamine\footnote{\url{https://enamine.net/}}, ChemBridge\footnote{\url{https://chembridge.com/}}, Life Chemicals\footnote{\url{https://lifechemicals.com/}}, Vitas-M\footnote{\url{https://vitasmlab.biz/}}, Maybridge\footnote{\url{https://chembridge.com/}}, Asinex\footnote{\url{https://www.asinex.com/}}, Eximed\footnote{\url{https://eximedlab.com/}}, InterBioScreen\footnote{\url{https://www.ibscreen.com/}}, Targetmol\footnote{\url{https://www.targetmol.com/}}, Princeton BioMolecular\footnote{\url{https://princetonbio.com/}}, Otava\footnote{\url{https://www.otava.com/}}, and Alinda Chemical\footnote{\url{https://www.alinda.ru/synthes_en.htm}}.
The raw aggregated library contains \textbf{727,519} fragments spanning \textbf{109,778} unique Murcko scaffolds~\citep{bemis1996properties}, providing broad coverage of drug-like chemical space.

To ensure that fragments selected for virtual screening are both drug-like and ADMET-compliant, we apply a two-stage filter. Each fragment is evaluated using the ADMET scoring function and drug-likeness scoring function described in Section~\ref{sec:admet-oracle}. Fragments are retained only if they satisfy:
\[
\text{ADMET score} > 2.5 \quad \text{and} \quad 
\text{QED score} > 0.7
\]
This yields a high-quality subset of \textbf{12,337} fragments that balance broad scaffold diversity with favorable drug-like and pharmacokinetic properties. These filtered fragments constitute the screening library used by the Chemist Agent during hit identification (Section~\ref{sec:hit}). Table~\ref{tab:fragment-library} summarizes the library statistics.

\begin{table}[!hbt]
\centering
\caption{Summary statistics of the MolLingo fragment library.}
\begin{tabular}{lc}
\toprule
\textbf{Property} & \textbf{Value} \\
\midrule
Total fragments & 727,519 \\
Unique scaffolds & 109,778 \\
Fragments passing ADMET \& QED filter & 12,337 \\
ADMET threshold & $> 2.5$ \\
QED threshold & $> 0.7$ \\
\bottomrule
\end{tabular}
\label{tab:fragment-library}
\end{table}

\section{ADMET Oracle Models and Ranking Function}
\label{sec:admet-oracle}

\subsection*{ADMET Oracle Models}
Accurate prediction of ADMET properties is essential for guiding molecule optimization toward drug-like candidates and filtering out compounds with unfavorable pharmacokinetic or toxicity profiles early in the discovery process. We train a suite of ADMET oracle models on tasks from the Therapeutics Data Commons (TDC)~\citep{huang2021therapeutics} benchmark, selecting tasks where data quality and quantity are sufficient to build reliable predictive models, as indicated by their leaderboard performance. Our selection follows the protocol of mCLM~\citep{edwards2025mclm}, covering six key ADMET endpoints that collectively assess toxicity, absorption, and permeability: DILI (Drug-Induced Liver Injury), hERG cardiotoxicity, AMES mutagenicity, HIA (Human Intestinal Absorption), Pgp (P-glycoprotein inhibition), and BBB (Blood-Brain Barrier Penetration). Model architectures and training procedures are adopted from mCLM~\citep{edwards2025mclm}, yielding similar results.

\subsection*{ADMET Score and Ranking Function}
\label{admet_ranking}

Given the predicted probabilities from each oracle model, we define the ADMET score of a molecule $m$ as the sum of its favorable outcome probabilities across all tasks:
\[ \text{ADMET}(m) = (1 - p_{\text{DILI}}) + (1 - p_{\text{AMES}}) + (1 - p_{\text{hERG}}) + (1 - p_{\text{Pgp}}) + p_{\text{HIA}} \]
where for toxicity and efflux endpoints (DILI, AMES, hERG, Pgp), a lower predicted probability is favorable, so we take the complement; for HIA, a higher predicted probability indicates better oral absorption and is thus directly included. BBB penetration is excluded from the aggregate score as its desirability is target-dependent (CNS penetration is required for neurological targets but undesirable for peripherally-acting drugs) and should be assessed separately based on the therapeutic indication. The ADMET score therefore ranges from 0 to 5, and is used as a hard filter by thresholding at $\text{ADMET}(m) > 0.5 \times= 2.5$, ensuring that all candidates forwarded to subsequent stages meet a minimum safety and pharmacokinetic standard.

For the construction of the screening fragment library, fragments are filtered by the composite criterion:
\[ \text{ADMET}(m) > 2.5 \quad \text{and} \quad \text{QED}(m) > 0.7 \]
ensuring that all fragments in the library meet a minimum standard of drug-likeness and pharmacokinetic safety prior to any target-specific screening.

Given a protein target, the Chemist Agent first attempts to retrieve known binders from ChEMBL~\citep{gaulton2012chembl}. If no relevant binding data is found, or if the user explicitly requests de novo design, the agent falls back to virtual screening over the curated fragment library. In this case, hit identification proceeds in two stages. First, DrugCLIP~\citep{gao2023drugclip} is used to rank all fragments in the screening library by their predicted binding affinity to the target protein, and the top $k$ fragments are retrieved. DrugCLIP is used here as a ranking model rather than an absolute affinity predictor, exploiting its ability to efficiently screen large fragment libraries without explicit docking. Second, the top $k$ retrieved fragments are passed to a molecular docking engine, which evaluates their binding poses and scores against the target protein structure. The final hits are selected based on docking scores, providing a geometrically grounded assessment of binding complementarity to serve as starting points for the hit-to-lead stage.

\section{Hit-to-Lead: Docking-Guided Hotspot Identification and LLM-Based Fragment Growing}
\label{appendix:h2l}
Algorithm~\ref{alg:h2l} details the hit-to-lead pipeline introduced in \S\ref{sec:h2l}. The core idea is to decompose fragment growing into three phases: (i) obtaining a geometrically accurate binding pose via molecular docking; (ii) identifying the most promising atomic positions for fragment extension (referred to as hotspots) by jointly analyzing the steric availability and protein environment around each heavy atom of the docked ligand; and (iii) grounding the LLM's fragment proposals in the three-dimensional biological context of each hotspot, enabling the Chemist Agent to reason about which building blocks are most likely to form favorable interactions with the surrounding residues. This design ensures that fragment growing is not a blind generative process, but a structurally informed and chemically reasoned decision, analogous to how a medicinal chemist would analyze a docking pose before proposing a synthetic modification.

\begin{algorithm}[!hbt]
\caption{Docking-Guided Hotspot Identification and LLM-Based Fragment Growing}
\label{alg:h2l}
\begin{algorithmic}[1]
\Require Protein target $P$, hit fragment $F$, contact threshold $d_c = 7.0$ \AA, maximum hotspots $k=5$
\Ensure Ranked hotspots $\mathcal{H}$ and proposed fragment growths $\{\mathcal{B}_i\}_{i=1}^k$

\Statex \textbf{Phase 1: Molecular Docking}
\State Dock fragment $F$ into the binding site of protein $P$
\State Save the docking pose as receptor structure $R$ and docked ligand structure $L$

\Statex \textbf{Phase 2: Hotspot Identification}
\State Parse heavy atoms (non-hydrogen) from $R$ and $L$, yielding $\mathcal{A}_R$ and $\mathcal{A}_L$
\State Build spatial index $\mathcal{N}$ over $\mathcal{A}_R$

\For{each heavy atom $a_i \in \mathcal{A}_L$}

 \State Identify neighboring residues:
 \[
 \mathcal{R}_i = \{ r \in R \mid \exists\, a \in r : \|a_i - a\| \leq d_c \}
 \]

 \State Estimate available volume around $a_i$
 \State Construct cubic grid $\mathcal{G}_i$ centered at $a_i$ (edge $5.0$ \AA, resolution $0.5$ \AA)

 \For{each grid point $g \in \mathcal{G}_i$}
 \State Compute:
 \[
 d_R(g) = \min_{a \in \mathcal{A}_R} \|g - a\|, \qquad
 d_L(g) = \min_{a \in \mathcal{A}_L} \|g - a\|
 \]
 \EndFor

 \State Count unoccupied grid points:
 \[
 n_i = \big|\{ g \in \mathcal{G}_i \mid d_R(g) > 2.2~\text{\AA} \wedge d_L(g) > 1.2~\text{\AA} \}\big|
 \]

 \State Compute available volume:
 \[
 V_i = n_i \times (0.5)^3
 \]

 \State Record hotspot:
 \[
 h_i = \big(a_i,\; \text{type}(a_i),\; V_i,\; \mathcal{R}_i \big)
 \]

\EndFor

\State Sort $\{h_i\}$ by $V_i$ in descending order
\State Select top-$k$ hotspots:
\[
\mathcal{H} = \{h_1, \ldots, h_k\}
\]

\Statex \textbf{Phase 3: LLM-Guided Fragment Growing}

\For{each hotspot $h_i \in \mathcal{H}$}

 \State Construct biological context prompt:
 \[
 \text{prompt}_i = \big(\text{type}(a_i),\; V_i,\; \mathcal{R}_i,\; F \big)
 \]

 \State Query LLM to propose building blocks $\mathcal{B}_i$ attachable at $a_i$

\EndFor

\State \Return proposed fragment growths:
\[
\{\mathcal{B}_i\}_{i=1}^{k}
\]

\end{algorithmic}
\end{algorithm}

\section{Benchmarking Results}
\label{benchmarking}

\subsection*{Comparison with DrugGEN on Hit Identification}
We benchmark MolLingo's hit identification pipeline against DrugGEN~\citep{unlu2025target}, a GAN-based molecular generation model, on two protein targets: CDK2 and AKT1. DrugGEN requires retraining a separate model for each target, making large-scale evaluation impractical; we therefore restrict this comparison to these two well-studied oncology targets for which DrugGEN provides top hit results directly, without the need for retraining on our end.

For a fair comparison, both methods leverage ChEMBL~\citep{gaulton2012chembl} as the source of known binders for each target, though in different ways: MolLingo uses ChEMBL binders as a hit screening library, ranking and filtering candidates following the procedure described in Appendix~\ref{admet_ranking}, while DrugGEN uses ChEMBL binders as training data for its target-specific GAN model. For each target, we select the top 30 hits from each method and evaluate them across six metrics: docking score, QED, synthetic accessibility (SA), Lipinski's rule of five, Veber's oral bioavailability criteria, and PAINS filter pass rate.

Table~\ref{table:drug_generation_bolded} reports the results. MolLingo outperforms DrugGEN on docking score for both targets, achieving $-9.308$ vs. $-8.527$ kcal/mol on CDK2 and $-6.368$ vs. $-5.797$ kcal/mol on AKT1, with lower standard deviation in both cases, indicating more consistent hit quality. MolLingo also achieves higher QED and better synthetic accessibility scores on both targets, suggesting that the retrieved hits are not only more potent but also more drug-like and easier to synthesize. Lipinski, Veber, and PAINS scores are comparable between the two methods, with MolLingo performing marginally better or on par across both targets. Notably, MolLingo achieves these results without target-specific retraining, making it substantially more practical for novel or understudied targets where retraining data may be scarce.

\begin{table*}[!hbt]
\centering
\caption{Comparison of MolLingo and DrugGEN on hit identification for CDK2 and AKT1. Results are reported as mean $\pm$ standard deviation over the top 30 hits for each method. Bold values indicate the better performing method for each metric. ($\downarrow$) indicates lower is better; ($\uparrow$) indicates higher is better.}
\label{table:drug_generation_bolded}

\setlength{\tabcolsep}{3pt}

\resizebox{\textwidth}{!}{
\begin{tabular}{lcccccccccccc}
\toprule
& \multicolumn{2}{c}{\textbf{Docking Score ($\downarrow$)}} & \multicolumn{2}{c}{\textbf{QED ($\uparrow$)}} & \multicolumn{2}{c}{\textbf{SA ($\downarrow$)}} & \multicolumn{2}{c}{\textbf{Lipinski ($\uparrow$)}} & \multicolumn{2}{c}{\textbf{Veber ($\uparrow$)}} & \multicolumn{2}{c}{\textbf{PAINS ($\uparrow$)}} \\
\cmidrule(lr){2-3} \cmidrule(lr){4-5} \cmidrule(lr){6-7} \cmidrule(lr){8-9} \cmidrule(lr){10-11} \cmidrule(lr){12-13}
& \textbf{ MolLingo} & \textbf{DrugGEN} & \textbf{ MolLingo} & \textbf{DrugGEN} & \textbf{ MolLingo} & \textbf{DrugGEN} & \textbf{ MolLingo} & \textbf{DrugGEN} & \textbf{ MolLingo} & \textbf{DrugGEN} & \textbf{ MolLingo} & \textbf{DrugGEN} \\
\midrule

\textbf{CDK2} 
& \makecell{$\mathbf{-9.308}$ \\ $\mathbf{\pm 0.424}$} 
& \makecell{$-8.527$ \\ $\pm 0.714$} 
& \makecell{$\mathbf{0.63}$ \\ $\mathbf{\pm 0.095}$} 
& \makecell{$0.518$ \\ $\pm 0.111$} 
& \makecell{$\mathbf{2.848}$ \\ $\mathbf{\pm 0.393}$} 
& \makecell{$3.058$ \\ $\pm 0.487$} 
& \makecell{$3.926$ \\ $\pm 0.279$} 
& \makecell{$\mathbf{3.955}$ \\ $\mathbf{\pm 0.208}$} 
& \makecell{$\mathbf{2.0}$ \\ $\mathbf{\pm 0.0}$} 
& \makecell{$\mathbf{2.0}$ \\ $\mathbf{\pm 0.0}$} 
& \makecell{$\mathbf{0.913}$ \\ $\mathbf{\pm 0.282}$} 
& \makecell{$0.864$ \\ $\pm 0.343$} \\

\midrule

\textbf{AKT1} 
& \makecell{$\mathbf{-6.368}$ \\ $\mathbf{\pm 0.251}$} 
& \makecell{$-5.797$ \\ $\pm 0.313$} 
& \makecell{$\mathbf{0.658}$ \\ $\mathbf{\pm 0.038}$} 
& \makecell{$0.567$ \\ $\pm 0.109$} 
& \makecell{$\mathbf{3.488}$ \\ $\mathbf{\pm 0.194}$} 
& \makecell{$3.741$ \\ $\pm 0.382$} 
& \makecell{$\mathbf{3.991}$ \\ $\mathbf{\pm 0.012}$} 
& \makecell{$3.862$ \\ $\pm 0.345$} 
& \makecell{$\mathbf{2.0}$ \\ $\mathbf{\pm 0.0}$} 
& \makecell{$1.966$ \\ $\pm 0.182$} 
& \makecell{$\mathbf{1.0}$ \\ $\mathbf{\pm 0.0}$} 
& \makecell{$\mathbf{1.0}$ \\ $\mathbf{\pm 0.0}$} \\

\bottomrule
\end{tabular}
}
\end{table*}

\subsection*{hERG Optimization Results}
\label{app:herg-results}

Table~\ref{table:herg} reports results on the hERG benchmark, 
evaluating the set of models on 33 correctly predicted hERG-positive 
withdrawn drugs. The results are consistent with and further 
strengthen the findings from the DILI benchmark. MolLingo with GPT-5.4 
as the base LLM achieves the highest improvement (13.03\%) and the 
highest similarity (36.6\%) among all models, with 100\% validity. 
The block-based representation advantage is again clearly isolated 
by the direct comparison with the base LLM: GPT-5.4 operating on 
raw SMILES yields a large negative improvement ($-32.1\%$), while 
 MolLingo using the same underlying model achieves the best optimization 
performance. Notably, the new results include MolLingo variants built 
on Claude-4.6-Sonnet and Gemini-3-Pro as base LLMs, both of which 
show a consistent pattern: the raw LLM baseline yields low or 
negative improvement, while the MolLingo variant with block-based 
representation substantially recovers and improves performance, 
with markedly higher similarity scores (54.97\% and 51.19\% 
respectively) compared to their raw SMILES counterparts (16.94\% 
and 14.74\%). This cross-model consistency strongly suggests that 
the benefit of block-based representation is model-agnostic rather 
than specific to any particular LLM backbone. Qwen2-family models 
again show competitive improvement but at substantially lower 
similarity, confirming that raw SMILES generation produces less 
conservative edits regardless of model scale.

\paragraph{Excluded Baselines.} Several additional models were considered for benchmarking but excluded for the following reasons. LLaMA-2-7B and Gemini-3.1-Flash did not follow the optimization instructions reliably. In particular, Gemini-3.1-Flash consistently failed to perform structural modifications: instead of proposing modified molecules, it analyzed the input molecule's constituent building blocks and predicted possible properties, without making any changes to the molecular structure. As reliable instruction following is a prerequisite for meaningful comparison, these models were excluded from the benchmark.

We also attempted to benchmark against closely related molecular optimization methods, specifically RePO~\citep{li2026reference} and ReMOL~\citep{wang2018remol}. ReMOL was excluded as the authors do not provide a public GitHub repository, making reproduction infeasible. RePO was excluded as no pretrained model checkpoints were available at the time of this work, and training from scratch was beyond the scope of this comparison.

\begin{table*}[!hbt]
\centering
\caption{hERG Comparison of model performance across multiple 
chemical generation metrics. Bold values represent the 
top-performing model for each metric. Values in parentheses 
represent the standard deviation (std). ($\downarrow$) indicates 
lower is better; ($\uparrow$) indicates higher is better.}
\label{table:herg}
\setlength{\tabcolsep}{3pt}
\resizebox{\textwidth}{!}{
\begin{tabular}{lcccccc}
\toprule
\textbf{Model} & 
\textbf{\makecell{Improvement (\%)\\($\uparrow$)}} & 
\textbf{\makecell{Similarity (\%)\\($\uparrow$)}} & 
\textbf{\makecell{Imprv $\times$ Sim\\($\uparrow$)}} & 
\textbf{\makecell{Drug-likeness\\($\uparrow$)}} & 
\textbf{\makecell{Synth\\($\downarrow$)}} & 
\textbf{\makecell{Validity (\%)\\($\uparrow$)}} \\
\midrule
GPT-5.4~\citep{achiam2023gpt}
& -0.571 (0.526)
& 15.889 (12.740)
& -9.072
& 0.685 (0.182)
& 3.005 (0.652)
& 96.364 \\
MolLingo (GPT-5.4 base)
& $\mathbf{4.050 \ (1.458)}$
& $\mathbf{42.878 \ (11.081)}$
& 173.736
& 0.492 (0.201)
& 3.276 (0.762)
& $\mathbf{100}$ \\
Claude-4.6-Sonnet~\citep{anthropic_claude_2024}
& 0.480 (0.690)
& 22.119 (10.982)
& 10.617
& $\mathbf{0.706 \ (0.153)}$
& $\mathbf{2.645 \ (0.466)}$
& 93.939 \\
MolLingo (Claude-4.6-Sonnet base)
& 5.370 (3.228)
& 54.966 (20.248)
& $\mathbf{295.467}$
& 0.549 (0.096)
& 3.289 (0.323)
& $\mathbf{100}$ \\
Gemini-3-Pro~\citep{google_gemini_2024}
& -0.562 (1.396)
& 16.085 (13.948)
& -9.040
& 0.632 (0.200)
& 3.119 (0.670)
& 98.589 \\
MolLingo (Gemini-3-Pro base)
& 0.675 (3.218)
& 51.186 (17.258)
& 34.526
& 0.553 (0.058)
& 3.196 (0.277)
& $\mathbf{100}$ \\
Qwen2-7B-Instruct~\citep{qwen2}
& 2.974 (0.028)
& 14.276 (5.335)
& 42.476
& 0.568 (0.033)
& 2.707 (0.193)
& $\mathbf{100}$ \\
Qwen2-14B-Instruct~\citep{qwen2}
& 2.974 (0.028)
& 14.276 (5.335)
& 42.476
& 0.568 (0.033)
& 2.707 (0.193)
& $\mathbf{100}$ \\
\bottomrule
\end{tabular}
}
\end{table*}

\subsection*{Molecule Optimization on TOMG-Bench}
We further evaluate MolLingo's molecule optimization capability on TOMG-Bench~\citep{li2024tomg}, a standardized benchmark for text-guided molecular optimization covering three property objectives: LogP (lipophilicity), QED (drug-likeness), and MR (molecular refractivity). We report three metrics for each objective: Success Rate (SR), molecular Similarity (Sim), and their product SR$\times$Sim, which jointly captures optimization effectiveness and structural conservation.

Table~\ref{table:large_model_comparison}, adopted from~\cite{li2024tomg}, compares MolLingo against a wide range of baselines spanning general-purpose frontier LLMs, chemistry-specialized models (MolT5, BioT5), fine-tuned instruction models (OpenMolIns), and the specialized molecular optimization method RePO~\citep{li2026reference}. MolLingo achieves the highest SR, Sim, and SR$\times$Sim across all three objectives (except for QED-Sim), outperforming all baselines by a substantial margin. On LogP, MolLingo achieves an SR$\times$Sim of 0.706, compared to 0.568 for the next best LLM (Claude-3.5) and 0.297 for RePO. On QED, the most challenging objective, where most models show significant drops in performance, MolLingo achieves an SR$\times$Sim of 0.418, more than doubling RePO's score of 0.236 and substantially outperforming all frontier LLMs. On MR, MolLingo achieves 0.752, again the highest across all methods.

Notably, MolLingo's advantage is most pronounced on the SR metric, reflecting the effectiveness of block-based modifications in producing valid, property-satisfying molecules. At the same time, MolLingo maintains the highest similarity scores across all objectives, confirming that its block-level edits are conservative and structurally coherent. The comparison with RePO is particularly informative: despite RePO being a specialized molecular optimization method, MolLingo outperforms it on both SR and SR$\times$Sim across all three objectives, while operating as a general-purpose agent without task-specific training. Among general-purpose LLMs, GPT-5.4 is the strongest baseline but still falls short of MolLingo on all metrics, again highlighting the contribution of the block-based representation over raw SMILES optimization.

\begin{table*}[!hbt]
\centering
\caption{Performance comparison on TOMG-Bench across LogP, QED, and MR optimization objectives. SR denotes Success Rate, Sim denotes Tanimoto Similarity, and SR$\times$Sim represents their product, jointly capturing optimization effectiveness and structural conservation. Bold values indicate the best performing model for each metric. Except for MolLingo, GPT-5.4, and RePO results, all other results are adopted from~\citet{li2024tomg}.}
\label{table:large_model_comparison}

\setlength{\tabcolsep}{5pt} 

\resizebox{\textwidth}{!}{
\begin{tabular}{lccccccccc}
\toprule
\textbf{Model Name} & \multicolumn{3}{c}{\textbf{LogP}} & \multicolumn{3}{c}{\textbf{QED}} & \multicolumn{3}{c}{\textbf{MR}} \\
\cmidrule(lr){2-4} \cmidrule(lr){5-7} \cmidrule(lr){8-10}
& \textbf{SR} & \textbf{Sim} & \textbf{SR$\times$Sim} & \textbf{SR} & \textbf{Sim} & \textbf{SR$\times$Sim} & \textbf{SR} & \textbf{Sim} & \textbf{SR$\times$Sim} \\
\midrule

\textbf{ MolLingo (GPT-5.4 base)} & $\mathbf{0.968}$ & $\mathbf{0.730}$ & $\mathbf{0.706}$ & $\mathbf{0.765}$ & 0.546 & $\mathbf{0.418}$ & $\mathbf{0.998}$ & $\mathbf{0.753}$ & $\mathbf{0.752}$ \\
GPT-5.4 & 0.871 & 0.612 & 0.533 & 0.432 & 0.284 & 0.123 & 0.903 & 0.691 & 0.624 \\
RePO & 0.415 & 0.715 & 0.297 & 0.312 & $\mathbf{0.756}$ & 0.236 & 0.399 & 0.736 & 0.294 \\
GPT-4o & 0.719 & 0.659 & 0.474 & 0.3952 & 0.618 & 0.244 & 0.6864 & 0.642 & 0.44 \\
GPT-4-turbo & 0.7662 & 0.6984 & 0.535 & 0.3946 & 0.6587 & 0.256 & 0.7388 & 0.6821 & 0.504 \\
GPT-3.5-turbo & 0.4048 & 0.6327 & 0.256 & 0.3316 & 0.5635 & 0.187 & 0.412 & 0.6263 & 0.258 \\
Claude-3.5 & 0.797 & 0.7124 & 0.568 & 0.5361 & 0.7042 & 0.378 & 0.6962 & 0.7112 & 0.495 \\
Claude-3 & 0.7984 & 0.6067 & 0.484 & 0.4678 & 0.5855 & 0.274 & 0.6094 & 0.6398 & 0.39 \\
Gemini-1.5-pro & 0.7712 & 0.7022 & 0.542 & 0.4704 & 0.6077 & 0.286 & 0.7876 & 0.6744 & 0.531 \\
Llama3-70B-Instruct & 0.5984 & 0.6028 & 0.361 & 0.2774 & 0.4828 & 0.134 & 0.5684 & 0.6032 & 0.343 \\
Llama3-8B-Instruct & 0.4642 & 0.3658 & 0.17 & 0.2568 & 0.4547 & 0.117 & 0.4332 & 0.4793 & 0.208 \\
Llama3.1-8B-Instruct & 0.399 & 0.4235 & 0.169 & 0.2655 & 0.4499 & 0.119 & 0.4164 & 0.483 & 0.201 \\
Mistral-7B-Instruct-v0.2 & 0.222 & 0.4501 & 0.1 & 0.121 & 0.3244 & 0.039 & 0.1908 & 0.2578 & 0.049 \\
Qwen2-7B-Instruct & 0 & 0.2923 & 0 & 0 & 0 & 0 & 0.0002 & 0.4123 & 0 \\
Yi-1.5-9B & 0.2884 & 0.5461 & 0.157 & 0.1064 & 0.6596 & 0.07 & 0.205 & 0.3724 & 0.076 \\
Chatglm-9B & 0.3666 & 0.6902 & 0.253 & 0.1832 & 0.6506 & 0.119 & 0.3514 & 0.682 & 0.24 \\
Llama-3.2-1B-Instruct & 0.0644 & 0.5055 & 0.033 & 0.0714 & 0.4757 & 0.034 & 0.0822 & 0.441 & 0.036 \\

\midrule
MolT5-small & 0.2158 & 0.1052 & 0.023 & 0.2214 & 0.1031 & 0.023 & 0.2316 & 0.1011 & 0.023 \\
MolT5-base & 0.2074 & 0.1051 & 0.022 & 0.2358 & 0.1054 & 0.025 & 0.1856 & 0.1073 & 0.02 \\
MolT5-large & 0.4244 & 0.1015 & 0.043 & 0.4654 & 0.119 & 0.055 & 0.4496 & 0.1072 & 0.048 \\
BioT5-base & 0.5158 & 0.1526 & 0.079 & 0.5068 & 0.158 & 0.08 & 0.506 & 0.1597 & 0.081 \\

\midrule
OpenMolIns-large (Llama-3.2-1B) & 0.2898 & 0.5951 & 0.172 & 0.1996 & 0.5849 & 0.117 & 0.2644 & 0.5956 & 0.157 \\
OpenMolIns-large (Llama-3.1-8B) & 0.8054 & 0.6678 & 0.538 & 0.5224 & 0.6398 & 0.334 & 0.7122 & 0.6548 & 0.466 \\
OpenMolIns-light (Galactica-125M) & 0.3202 & 0.6547 & 0.21 & 0.269 & 0.6521 & 0.175 & 0.3508 & 0.6435 & 0.226 \\
OpenMolIns-small (Galactica-125M) & 0.4172 & 0.642 & 0.268 & 0.2956 & 0.6385 & 0.189 & 0.3958 & 0.6452 & 0.255 \\
OpenMolIns-medium (Gal-125M) & 0.5904 & 0.5812 & 0.343 & 0.4608 & 0.5859 & 0.27 & 0.5874 & 0.5873 & 0.345 \\
OpenMolIns-large (Galactica-125M) & 0.6454 & 0.5927 & 0.383 & 0.495 & 0.5962 & 0.295 & 0.6388 & 0.5973 & 0.382 \\
OpenMolIns-xlarge (Gal-125M) & 0.7362 & 0.5744 & 0.423 & 0.5786 & 0.5677 & 0.328 & 0.7124 & 0.5697 & 0.406 \\

\bottomrule
\end{tabular}
}
\end{table*}

\subsection*{Attention Probing: Block-Based vs. Raw SMILES Representation}
\label{app:attention-probing}

To validate that block-based SMILES genuinely improves structural alignment with LLM semantic space, we probe the attention mechanism of Qwen2-Instruct-7B~\citep{qwen2} under two input conditions: raw SMILES and block-based SMILES paired with common names. Concretely, for the same molecule (imatinib), we construct two sentences using the same template and query:

\begin{tcolorbox}[
 colback=gray!5,
 colframe=gray!40,
 title=Sentence 1: Raw SMILES input,
 fonttitle=\small\bfseries,
 fontupper=\small\ttfamily,
 breakable
]
Cc1ccc(NC(=O)c2ccc(CN3CCN(C)CC3)cc2)cc1Nc1nccc(-c2cccnc2)n1. 
The building block responsible for hinge binder is:
\end{tcolorbox}

\begin{tcolorbox}[
 colback=gray!5,
 colframe=gray!40,
 title=Sentence 2: Block-based SMILES with common names,
 fonttitle=\small\bfseries,
 fontupper=\small\ttfamily,
 breakable
]
1,4-Dimethylpiperazine [1*]CN1CCN(C)CC1, benzene 
[1*]c1ccc([2*])cc1, formamide [1*]NC([2*])=O, toluene 
[2*]c1ccc(C)c([1*])c1, 2-AMINOPYRIMIDINE 
[2*]Nc1nccc([1*])n1, PYRIDINE [2*]c1cccnc1. The building 
block responsible for hinge binder is:
\end{tcolorbox}

We compute the average attention between all tokens belonging to each block and the tokens of the query phrase ``hinge binder'', and visualize the resulting attention distribution across blocks.

Fig.~\ref{fig:attention_probing} shows the attention distributions across transformer layers for both input conditions. When the input is raw SMILES, attention scores are distributed uniformly across all tokens with no discernible structure, reflecting the inability of the LLM to associate character-level SMILES tokens with chemically meaningful substructures. In contrast, block-based SMILES with common names produces strongly localized attention patterns: the highest attention weights are consistently observed between the tokens of the functionally relevant block and the words describing its associated biological function, while attention to unrelated blocks remains low. This confirms that block-based representation enables the LLM to form explicit and interpretable associations between molecular substructures and their chemical functions, a capability that is entirely absent with raw SMILES input, regardless of the model's underlying chemical knowledge.

\begin{figure}[!hbt]
\centering
\includegraphics[width=0.81\linewidth]{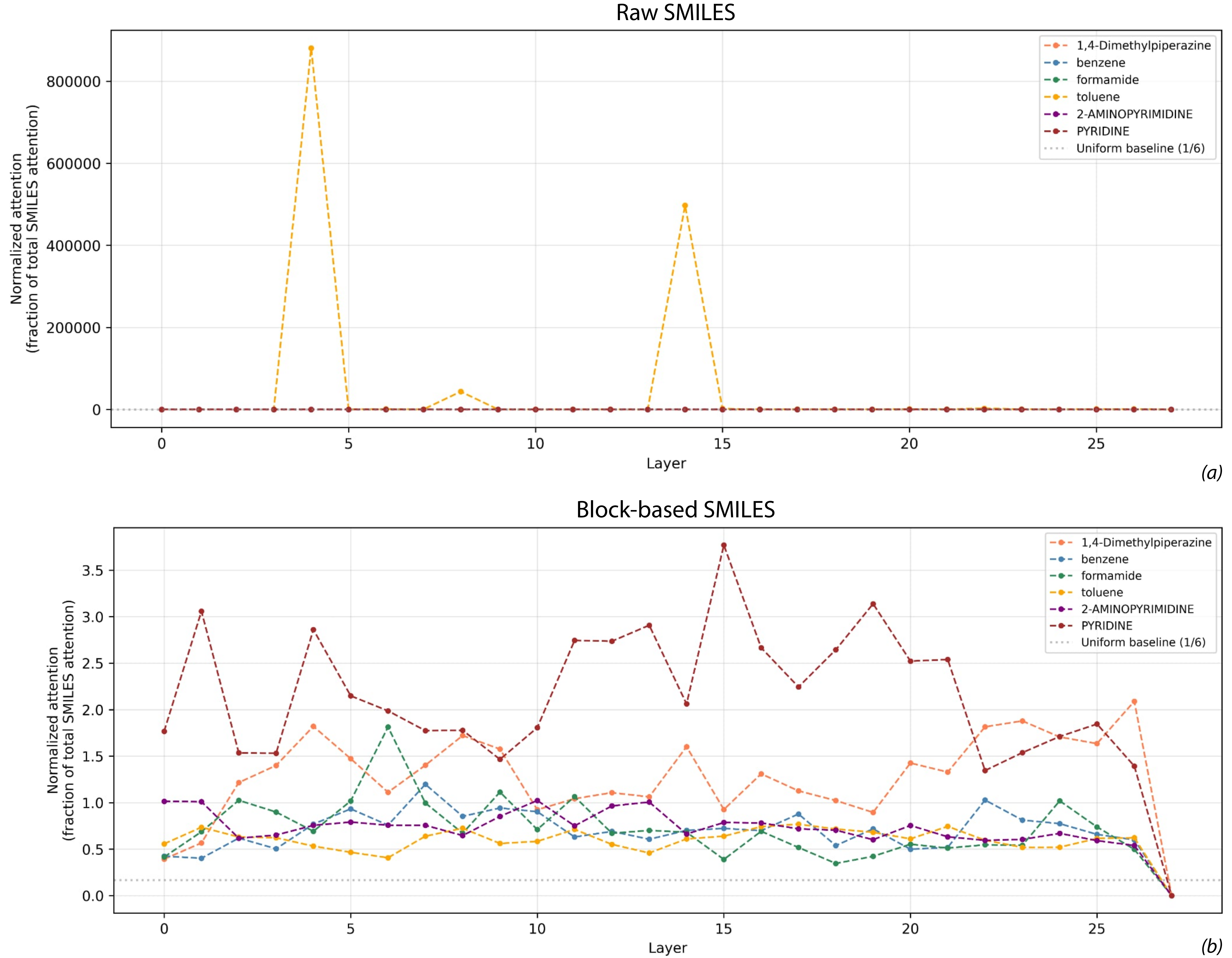}
\caption{Attention heatmaps of Qwen2-Instruct-7B for raw SMILES (top) and block-based SMILES with common names (bottom). Each line represents the average attention between tokens of a molecular block and tokens describing a biological function. Raw SMILES produces random, unstructured attention, while block-based representation yields strongly localized attention at the functionally relevant block.}
\label{fig:attention_probing}
\end{figure}

\section{Biological Target Identification}
\label{app:target-identification}

We developed a computational pipeline to identify protein targets and their associated therapeutic agents by integrating a large language model (LLM) with retrieval-augmented generation (RAG). RAG combines the generative capabilities of an LLM with a retrieval step that grounds outputs in external evidence, improving the accuracy and relevance of extracted information for tasks that require domain-specific or up-to-date knowledge~\citep{lewis2020retrieval}. The system constructs context-specific queries that incorporate disease, disease subtype, relevant biological pathways, and known gene mutations to retrieve pertinent publications. The LLM then synthesizes the retrieved documents into structured outputs, which are standardized to protein identifiers and aggregated across sources. Aggregation allows quantification of how frequently each protein target appears in the context of disease subtype, pathway, and gene mutation, highlighting targets that are specifically relevant to particular patient subgroups.

Most existing computational approaches identify biological targets based only on the disease label, without considering patient- or subgroup-specific information, when querying literature or extracting relevant proteins.
For example, in non-small cell lung cancer, EGFR mutations indicate sensitivity to EGFR-targeted therapies, while ALK fusions highlight ALK inhibitors, and KRAS mutations represent a distinct molecular group with different therapeutic considerations~\citep{lynch2004activating}. A system relying solely on the disease label may conflate these subgroups, leading to inclusion of targets that are not actionable for all patients. Incorporating subtype, pathway, and mutation information ensures that extracted targets are context-specific, providing a more precise and biologically meaningful representation.

The structured output supports downstream analyses such as target prioritization, integration with predictive models, or decision support systems. By combining evidence retrieval with generative modeling and systematic aggregation, this method enables reproducible identification of protein targets that are relevant to both disease and molecular context.

We evaluate the biological target identification pipeline on a diverse set of 30 diseases spanning oncology, neurology, cardiovascular disease, metabolic disorders, and autoimmune conditions. The pipeline achieves 100\% accuracy on 28 out of 30 diseases, meaning that for each of these diseases, all extracted information, including the primary therapeutic protein target, its UniProt identifier, and the associated reference, is correct. The two remaining cases are considered unsuccessful not because the returned targets are incorrect, but because the identified proteins lack sufficiently established evidence as primary therapeutic targets: their roles in the respective diseases remain under active investigation, and no strong consensus exists in the literature at the time of evaluation.

The two cases where the pipeline did not return the expected primary target are not outright failures but rather reflect the inherent complexity and ambiguity of target-disease associations in the literature.
Table~\ref{tab:target-identification} reports the full results of 
the Literature Agent on all 30 diseases, including the identified 
protein target, UniProt ID, and supporting reference.

\begin{table*}[!hbt]
\centering
\caption{Biological target identification results across 30 diverse 
diseases. For diseases with multiple identified targets, all 
returned targets are listed. $^*$ indicates the two cases where 
the primary expected target was not returned as the top result.}
\label{tab:target-identification}
\setlength{\tabcolsep}{6pt}
\resizebox{\textwidth}{!}{
\begin{tabular}{llll}
\toprule
\textbf{Disease} & \textbf{Protein Target} & \textbf{UniProt ID} & 
\textbf{Reference} \\
\midrule
Breast cancer & Estrogen Receptor Alpha (ER$\alpha$) & P03372 & 
Targeting the Estrogen Receptor\ldots \\
& HER2 (ERBB2) & P04626 & Targeting HER2-positive breast cancer\ldots \\
Non-small cell lung cancer & KRAS & P01116 & Protein Kinase C as a 
Therapeutic Target\ldots \\
Melanoma & BRAF & P15056 & BRAF as therapeutic target in melanoma \\
Chronic myeloid leukemia & BCR-ABL1 & P00519 & Effects of a selective 
inhibitor of the Abl\ldots \\
Colon cancer & VEGF & P15692 & Targeted Therapy Drugs for Colorectal 
Cancer \\
Prostate cancer & Androgen Receptor & P10275 & Protein Degradation 
Inducers of Androgen Receptor\ldots \\
Acute myeloid leukemia & FLT3 & P36888 & Midostaurin prolongs survival 
in FLT3-mutated AML \\
Rheumatoid arthritis & TNF-$\alpha$ & P01375 & Targeting rheumatoid 
arthritis\ldots \\
Psoriasis & Interleukin-17A & Q16552 & Psoriasis: rationale for 
targeting IL-17 \\
Multiple sclerosis$^*$ & MOG & Q16653 & Identification of a pathogenic 
antibody response\ldots \\
& MBP & P02686 & Multiple Sclerosis: An Important Role\ldots \\
Alzheimer's disease & Amyloid-beta (APP) & P05067 & Amyloid 
beta-protein assembly\ldots \\
Parkinson's disease & Alpha-synuclein & P37840 & Targeting the 
alpha-synuclein protein\ldots \\
Type 2 diabetes & PTP1B & P18031 & PTP1B as a Therapeutic Target\ldots \\
Hypercholesterolemia & PCSK9 & Q8NBP7 & Targeting PCSK9\ldots \\
Hypertension$^*$ & Angiotensin II receptor type 1 & P30556 & 
Angiotensin II receptor blockers\ldots \\
& Endothelin-1 & P05305 & Endothelin-1 is a potent vasoconstrictor\ldots \\
& ERAP1 & Q9NZ08 & Potential Therapeutic Drug Targets\ldots \\
& ACVRL1 & Q04771 & Potential Therapeutic Drug Targets\ldots \\
& Renin & P00797 & Identification of novel proteomic biomarkers\ldots \\
Asthma & IL-4 receptor alpha chain & P24394 & Monoclonal antibodies 
targeting IL4R\ldots \\
COPD & Alpha-1 antitrypsin & P01009 & Role of elastases in the 
pathogenesis of COPD\ldots \\
HIV & CD4 & P01730 & HIV Immune Cell Receptors\ldots \\
Hepatitis C & NS5A & Q9YKI6 & Phosphorylation of hepatitis C virus NS5A\ldots \\
Influenza & Hemagglutinin & P03437 & Hemagglutinin protein of Asian 
strains\ldots \\
Malaria & PfATP6 & E1CC99 & Artemisinins target the SERCA\ldots \\
Tuberculosis & InhA & P9WIR7 & The progress of M. tuberculosis drug 
targets \\
Crohn's disease & NOD2 & Q9HC29 & NOD2 and its role in Crohn's 
disease \\
Ulcerative colitis & Interleukin-23 & P78558 & Effectiveness and 
Safety of Selective IL-23\ldots \\
Osteoporosis & RANKL & O14788 & RANKL as a target for osteoporosis \\
Sickle cell disease & Hemoglobin subunit beta & P68871 & Silencing 
of BCL11A\ldots \\
& BCL11A & Q9H165 & Silencing of BCL11A\ldots \\
Epilepsy & NaV1.1 (SCN1A) & Q99250 & SCN1A/NaV1.1 
channelopathies\ldots \\
Migraine & CGRP Receptor (CALCRL + RAMP1) & Q16602 & 
Targeting enkephalins\ldots \\
Obesity & FGF-21 & P54652 & Adipokines as Clinically Relevant 
Therapeutic Targets\ldots \\
COVID-19 & Spike glycoprotein & P0DTC2 & SARS-CoV-2 Proteins\ldots \\
\bottomrule
\end{tabular}
}
\end{table*}

\section{Shared Memory Module}
\label{app:memory}

The shared memory module serves as the central coordination mechanism of MolLingo, enabling persistent context and iterative reasoning across agents and pipeline stages. We describe its design in terms of schema, storage, retrieval, and lifecycle management.

\subsection*{Schema and Content Types}

The shared memory maintains two categories of information. \textbf{Structured records} are stored as typed key-value entries with a defined schema, including: target annotations (disease name, protein name, UniProt ID, relevant pathways and mutations), candidate molecules (SMILES, block-based representation, docking scores, ADMET scores, QED, SA), and pipeline state (current stage, completed subgoals, pending tasks). \textbf{Unstructured context} is stored as free-text entries, including reasoning traces from each agent, literature summaries retrieved by the Literature Agent, design rationales for proposed modifications, and human-in-the-loop feedback. Each entry is tagged with a timestamp, the agent that wrote it, and the pipeline stage at which it was created, enabling provenance tracking and stage-specific retrieval.

\subsection*{Retrieval Mechanism}

Agents access the shared memory through two retrieval modes. \textbf{Exact retrieval} is used for structured records, where agents query by key (e.g., UniProt ID, molecule SMILES) to retrieve specific entries. \textbf{Semantic retrieval} is used for unstructured context, where a query string is embedded and matched against stored entries using cosine similarity over dense vector representations, allowing agents to retrieve relevant reasoning traces or literature summaries without knowing their exact content. In practice, the Orchestrator primarily uses exact retrieval to track pipeline state, while the Literature Agent and Chemist Agent use semantic retrieval to access accumulated biological and chemical context.

\subsection*{Memory Size and Eviction}

The shared memory is designed to hold the full context of a single drug discovery session, which in practice involves a target and a moderate number of hits, leads, and candidate drugs (tens to hundreds). No eviction policy is applied within a session, as the total memory footprint remains manageable. When the research subject changes (for example, when the user switches to a new protein target or disease) the memory module supports a selective reset operation that clears all target-specific entries (candidate molecules, docking results, ADMET profiles, design rationales, and literature summaries associated with the previous target) while retaining session-level context such as user preferences, high-level design goals, and human-in-the-loop feedback. This selective reset prevents context contamination across different research subjects while preserving information that remains relevant across targets. A full memory reset is also available when starting an entirely new discovery campaign.

\subsection*{Read and Write Protocol}

All agents interact with the shared memory through a unified read/write interface. Write operations append new entries without overwriting existing ones, preserving the full history of the discovery process. Read operations return the most recent entry matching the query by default, with an option to retrieve the full history of entries for a given key. The Orchestrator additionally maintains a session log that records the sequence of agent calls, tool invocations, and memory operations, providing a complete audit trail of the discovery workflow that supports human-in-the-loop review and debugging.

\section{Orchestrator Agent: Decision-Making and Coordination}
\label{app:orchestrator}

The Orchestrator Agent follows a ReAct-style~\citep{yao2022react} decision-making loop, interleaving reasoning steps with tool invocations and memory operations to decompose and execute the drug discovery workflow. At each step, the Orchestrator is provided with the current task objective, the history of completed subgoals stored in shared memory, and the available agents and tools. It then reasons about the next action to take, invokes the appropriate agent or tool, observes the result, updates the shared memory, and repeats until the objective is satisfied.

\subsection*{Decision-Making Algorithm}

Algorithm~\ref{alg:orchestrator} describes the Orchestrator's decision-making loop. The Orchestrator is implemented as a prompted GPT-5.4 model with access to a fixed set of callable actions: invoking the Literature Agent, invoking the Chemist Agent, calling a tool directly, writing to or reading from shared memory, requesting human feedback, and terminating the session. At each step, the full task context, including the original objective, completed subgoals, and relevant memory entries, is included in the prompt, allowing the Orchestrator to make informed decisions without maintaining internal state.

\begin{algorithm}[t]
\caption{Orchestrator Decision-Making Loop}
\label{alg:orchestrator}
\begin{algorithmic}[1]
\Require Task objective $\mathcal{O}$, shared memory $\mathcal{M}$, 
available actions $\mathcal{A}$
\Ensure Completed drug discovery workflow

\State Decompose $\mathcal{O}$ into initial subgoal queue $\mathcal{Q}$
\While{$\mathcal{Q}$ is not empty}
 \State $q \gets$ \textsc{NextSubgoal}($\mathcal{Q}$)
 \State $c \gets$ \textsc{RetrieveContext}($\mathcal{M}$, $q$)
 \State Construct prompt $p \gets (q, c, \mathcal{A})$
 \State $a \gets$ \textsc{LLM}($p$) \Comment{Reason and select action}
 \If{$a$ = \textsc{InvokeLiteratureAgent}}
 \State $r \gets$ \textsc{LiteratureAgent}($q$)
 \ElsIf{$a$ = \textsc{InvokeChemistAgent}}
 \State $r \gets$ \textsc{ChemistAgent}($q$)
 \ElsIf{$a$ = \textsc{CallTool}($t$)}
 \State $r \gets$ \textsc{Tool}($t$, $q$)
 \ElsIf{$a$ = \textsc{RequestHumanFeedback}}
 \State $r \gets$ \textsc{HumanFeedback}($q$)
 \ElsIf{$a$ = \textsc{Terminate}}
 \State \textbf{break}
 \EndIf
 \State \textsc{WriteMemory}($\mathcal{M}$, $q$, $r$)
 \State $\mathcal{Q} \gets$ \textsc{UpdateSubgoals}($\mathcal{Q}$, $q$, $r$)
\EndWhile
\State \Return \textsc{RetrieveFinalCandidates}($\mathcal{M}$)
\end{algorithmic}
\end{algorithm}

\subsection*{Prompt Template}

The following prompt template is used at each step of the Orchestrator's decision-making loop. Fields in curly braces are filled in at runtime from the task context and shared memory.

\begin{tcolorbox}[
 colback=gray!5,
 colframe=gray!40,
 title=Orchestrator Prompt Template,
 fonttitle=\small\bfseries,
 fontupper=\small\ttfamily,
 breakable
]
You are the Orchestrator of MolLingo, an autonomous drug discovery system. Your role is to decompose the given objective into subgoals and coordinate the Literature Agent, Chemist Agent, and available tools to complete the drug discovery workflow.

Current objective: \{OBJECTIVE\}

Completed subgoals:
\{COMPLETED\_SUBGOALS\}

Relevant context from memory:
\{MEMORY\_CONTEXT\}

Available actions:
- invoke\_literature\_agent(task): retrieve and synthesize 
 biological knowledge for a given task or target.
- invoke\_chemist\_agent(task): perform molecular reasoning, 
 hit identification, fragment growing, or lead optimization.
- call\_tool(tool\_name, inputs): invoke a specific tool 
 (e.g., docking, ADMET prediction, ChEMBL retrieval).
- request\_human\_feedback(question): ask the user for 
 clarification, constraints, or high-level guidance.
- terminate(): end the session and return final candidates.

Based on the current state, reason step by step about what needs to be done next, then select the most appropriate action. Respond in the following format:

Reasoning: \{step-by-step reasoning\}
Action: \{selected action and inputs\}
\end{tcolorbox}

\noindent The Orchestrator's reasoning trace and selected actions are written to the shared memory at each step, providing a complete audit trail of the workflow and enabling human-in-the-loop review at any stage of the discovery process.


\section{ MolLingo's Toolset}
\label{app:toolset}

 MolLingo's agents are equipped with a curated set of tools that support each stage of the drug discovery pipeline. We describe the core tools below, grouped by function.

\paragraph{Molecular Docking}

Docking is performed using AutoDock Vina~\citep{trott2010autodock}. To define the binding site, the Literature Agent queries the RCSB Protein Data Bank~\citep{berman2000protein} for co-crystal structures of the target protein with a bound ligand, selecting the structure with the lowest resolution (highest quality) available. The co-crystallized ligand is removed and its binding pocket is used to define the docking search box. If no co-crystal structure is available, the pipeline falls back to a RCSB structure without a ligand and performs blind docking over the full protein surface. If no experimental structure exists, an AlphaFold~\citep{jumper2021highly} predicted structure is retrieved and blind docking is performed. This three-tier fallback strategy ensures that docking can be performed for any target, including novel or understudied proteins with no experimental structural data.

\paragraph{SMILES to Block-Based Representation}
Given a SMILES string, this tool tokenizes the molecule into a sequence of building blocks from the learned BFE vocabulary, following the procedure described in \S\ref{sec:smiles}. Each block is converted to a block-based SMILES string with wildcard attachment points and optionally paired with its common chemical name, producing the block-based representation used by the Chemist Agent for molecular reasoning and optimization.

\paragraph{Fragment Screening with DrugCLIP}
Given a protein target and the curated fragment library described in Appendix~\ref{sec:fragment-library}, this tool uses DrugCLIP~\citep{gao2023drugclip} to rank all fragments by predicted binding affinity to the target. DrugCLIP embeds both the protein and each fragment into a shared semantic space and retrieves the top $k$ fragments by similarity score, enabling efficient virtual screening without explicit docking. The retrieved fragments are subsequently passed to the docking tool for binding pose evaluation.

\paragraph{Context-Aware Fragment Growing}

Given a docked hit fragment and its binding pose, this tool identifies 
growth hotspots and queries the Chemist Agent to propose fragment 
extensions, following the procedure described in \S\ref{sec:h2l} and 
detailed in Algorithm~\ref{alg:h2l}.

\paragraph{ADMET Optimization}

Given a lead molecule and its ADMET profile, this tool queries the 
Chemist Agent to propose block-level modifications that improve ADMET 
properties while preserving structural similarity to the original 
molecule, following the procedure described in \S\ref{sec:lead}.

\paragraph{Supporting Tools}

The following utility tools support the Literature Agent and Chemist 
Agent throughout the pipeline:

\begin{itemize}

 \item \textbf{ChEMBL binder retrieval.} Queries ChEMBL~\citep{gaulton2012chembl} for known binders of a given protein target, returning compounds with reported binding affinity data.

 \item \textbf{Molecular identifier conversion.} A suite of conversion tools: \texttt{name2smiles} and \texttt{smiles2name} for interconversion between common names and SMILES strings; \texttt{uniprotid2name} and \texttt{name2uniprotid} for mapping between protein names and UniProt identifiers; and \texttt{disease2drugs} for retrieving approved drugs associated with a given disease indication.

 \item \textbf{Drug-likeness filtering and ADMET ranking.} Filters candidate molecules by drug-likeness criteria and ranks them by the composite ADMET score described in \S\ref{sec:admet-oracle}.

 \item \textbf{Butina clustering.} Clusters a set of retrieved binders using the Butina algorithm~\citep{butina1999unsupervised} based on Tanimoto similarity, and selects one representative molecule per cluster. This ensures that the hits forwarded to the hit-to-lead stage are chemically diverse rather than redundant.

 \item \textbf{Property scoring.} Computes standard molecular property scores including synthetic accessibility (SA)~\citep{ertl2009estimation}, quantitative estimate of drug-likeness (QED)~\citep{bickerton2012quantifying}, and Tanimoto similarity~\citep{tanimoto1958elementary} between pairs of molecules, used throughout the pipeline for filtering, ranking, and similarity-constrained optimization.

\end{itemize}


\section{Implementation Details}
\subsection*{Prompt Templates}
\label{app:prompt}

\subsubsection*{Molecule Optimization Prompt}

The following prompt template is used for all baseline LLM models 
in the ADMET optimization benchmark. The molecule SMILES and target 
ADMET property are filled in at inference time.

\begin{tcolorbox}[
 colback=gray!5,
 colframe=gray!40,
 title=Molecule Optimization Prompt,
 fonttitle=\small\bfseries,
 fontupper=\small\ttfamily,
 breakable
]
Given the molecule \{SMILES\}, proposestructural modifications 
aimed at reducing its risk of \{ADMET\_PROPERTY\} (e.g., liver injury 
for DILI, cardiotoxicity for hERG).

For each proposed analog, briefly explain the rationale behind the 
modification and how it may lower \{ADMET\_PROPERTY\} risk.

In the conclusion, provide five candidate molecules with potentially 
improved \{ADMET\_PROPERTY\} profiles, formatted as a JSON dictionary 
where each entry includes the SMILES string and a short description 
of the modification. Make sure all the proposed SMILES are valid.

Example:
\begin{verbatim}[!hbt]
{
 "candidate_1": {
 "smiles": "CC(C1=CC=CC=C1)N(C(=O)N)O",
 "modification": "Removed sulfur-containing heterocycle to reduce potential
 bioactivation and reactive metabolite formation."
 },
 ...
 "candidate_5": {
 "smiles": "CC(C1=CC2=CC=CC=C2O1)N(C(=O)N)O",
 "modification": "Replaced sulfur with oxygen to decrease likelihood of forming
 toxic sulfur metabolites."
 }
}
\end{verbatim}
\end{tcolorbox}

\noindent Note that for MolLingo, the molecule is provided in block-based SMILES format with common block names rather than raw SMILES, while the output format remains identical. This is the only difference between the MolLingo prompt and the baseline prompt, isolating the effect of molecular representation on optimization performance.

\subsubsection*{Hit-to-Lead Optimization Prompt}

The following prompt template is used for all baseline LLM models 
in the hit-to-lead benchmark. The molecule SMILES and protein target 
are filled in at inference time.

\begin{tcolorbox}[
 colback=gray!5,
 colframe=gray!40,
 title=Hit-to-Lead Optimization Prompt,
 fonttitle=\small\bfseries,
 fontupper=\small\ttfamily,
 breakable
]
You are an expert medicinal chemist. Given a hit molecule and a 
protein target, your task is to proposestructural modifications 
that improve the binding affinity of the molecule to the target 
protein, while maintaining sufficient structural similarity to the 
original molecule to preserve its pharmacophore.

Molecule SMILES: \{SMILES\}

Protein target: \{PROTEIN\_NAME\} (UniProt: \{UNIPROT\_ID\})

Guidelines:

- Propose modifications that are likely to form new favorable interactions (hydrogen bonds, hydrophobic contacts, or electrostatic interactions) with the target protein.

- Each modification should be conservative — the proposed analog should remain structurally similar to the original molecule (Tanimoto similarity > 0.4).

- Avoid modifications that introduce known toxicophores or violate Lipinski's rule of five.

- Each proposed molecule must be a valid, chemically synthesizable structure.

For each proposed analog, briefly explain the rationale behind the modification and how it may improve binding affinity to \{PROTEIN\_NAME\}.

In the conclusion, provide five candidate molecules, formatted as a JSON dictionary where each entry includes the SMILES string and a short description of the modification. Make sure all proposed SMILES are valid.

Example:
\begin{verbatim}
{
 "candidate_1": {
 "smiles": "CC(C1=CC=CC=C1)N(C(=O)N)O",
 "modification": "Added hydrogen bond donor at para position to form additional
 interaction with Asp residue in the binding site."
 },
 ...
 "candidate_5": {
 "smiles": "CC(C1=CC2=CC=CC=C2O1)N(C(=O)N)O",
 "modification": "Replaced phenyl with naphthalene to improve hydrophobic
 packing in the lipophilic pocket."
 }
}
\end{verbatim}
\end{tcolorbox}

\noindent For MolLingo, the molecule is provided in block-based SMILES format with common block names rather than raw SMILES, and the Chemist Agent is additionally provided with the docking pose, hotspot analysis, and biological context of the binding site as described in \S\ref{sec:h2l}. All other instructions remain identical, isolating the contribution of the structural context and molecular representation on optimization performance.

\subsection*{Hyperparameters and Pipeline Configuration}
Several key parameters in MolLingo's pipeline are user-configurable and can be adjusted based on available computational resources and time constraints. These include the number of protein targets returned per disease by the Literature Agent, the number of top hits retrieved from ChEMBL per target, the number of fragment modifications proposed per hotspot in the hit-to-lead stage, the number of structural modifications proposed per iteration in lead optimization, and the number of lead optimization iterations. Increasing these values broadens the search space and may yield better candidates at the cost of additional computation and LLM calls. For all experiments reported in this paper, we use the following settings: one top hit retrieved from ChEMBL per target, five fragment modifications proposed per hotspot in the hit-to-lead stage, 15 structural modifications proposed per lead optimization iteration, and two lead optimization iterations per molecule.

\textbf{Butina clustering.} Since ChEMBL binders for a given target often share common scaffolds, we apply the Butina algorithm~\citep{butina1999unsupervised} to cluster retrieved compounds by Tanimoto similarity using a distance threshold of 0.7, and select one representative molecule per cluster. This ensures that the hits forwarded to the hit-to-lead stage are chemically diverse rather than redundant.

\subsection*{List of Protein Targets for Hit-to-Lead Evaluation}
\label{app:proteins}

Table~\ref{tab:proteins} lists the 30 protein targets used in the hit-to-lead evaluation, spanning diverse therapeutic areas including oncology, epigenetics, metabolism, and apoptosis.

\begin{table*}[!hbt]
\centering
\caption{Protein targets used for hit-to-lead evaluation, grouped by biological function.}
\label{tab:proteins}

\setlength{\tabcolsep}{40pt}

\resizebox{\textwidth}{!}{
\begin{tabular}{llll}
\toprule
\textbf{UniProt ID} & \textbf{Gene Name} & \textbf{UniProt ID} & \textbf{Gene Name} \\
\midrule

\multicolumn{4}{l}{\textit{Kinases}} \\
P31749 & AKT1 & P28482 & MAPK1 \\
P27361 & MAPK3 & P06493 & CDK1 \\
P24941 & CDK2 & P11802 & CDK4 \\
P00519 & ABL1 & O60674 & JAK2 \\
P49841 & GSK3B & O14757 & CHEK1 \\

\midrule
\multicolumn{4}{l}{\textit{Epigenetic regulators}} \\
Q13547 & HDAC1 & Q92769 & HDAC2 \\
P26358 & DNMT1 & Q15910 & EZH2 \\
O60341 & KDM1A & & \\

\midrule
\multicolumn{4}{l}{\textit{Metabolic enzymes}} \\
P00374 & DHFR & P04818 & TYMS \\
P04035 & HMGCR & P49327 & FASN \\
O75874 & IDH1 & P48735 & IDH2 \\
P00338 & LDHA & P00558 & PGK1 \\

\midrule
\multicolumn{4}{l}{\textit{Apoptosis \& proteolysis}} \\
P42574 & CASP3 & Q14790 & CASP8 \\
P08253 & MMP2 & P14780 & MMP9 \\

\midrule
\multicolumn{4}{l}{\textit{Chaperones \& DNA repair}} \\
P07900 & HSP90AA1 & P08238 & HSP90AB1 \\
P09874 & PARP1 & & \\

\bottomrule
\end{tabular}
}

\end{table*}

\section{Use Cases}
\label{app:use-cases}

We demonstrate MolLingo's end-to-end drug discovery capabilities through three representative use cases, illustrating the system's ability to autonomously navigate the pipeline from disease input to optimized drug candidate.

\paragraph{Disease to Biological Target.} 
Fig.~\ref{fig:demo1} illustrates MolLingo's target identification capability. Given a disease name as input, the Literature Agent retrieves and synthesizes relevant biological knowledge to identify the most contextually appropriate protein target, returning its UniProt identifier and supporting evidence from the literature.

\paragraph{Protein to Lead.}
Fig.~\ref{fig:demo2} demonstrates the pipeline from a protein target 
to a lead compound. Given a target protein, the Chemist Agent 
retrieves known binders from ChEMBL, applies Butina clustering to 
select a diverse set of representative hits, and performs 
docking-guided fragment growing on the top-ranked hit to produce 
an elaborated lead molecule with improved predicted binding affinity.

\paragraph{Lead Optimization.}
Fig.~\ref{fig:demo3} shows MolLingo's lead optimization capability. 
Given a lead molecule and its ADMET profile, the Chemist Agent 
iteratively proposes block-level modifications to improve 
drug-likeness and reduce toxicity liabilities, with each proposed 
molecule evaluated by the ADMET oracle models and the reasoning 
trace stored in shared memory for iterative refinement.

\begin{figure}[!hbt]
\centering
\includegraphics[width=\linewidth]{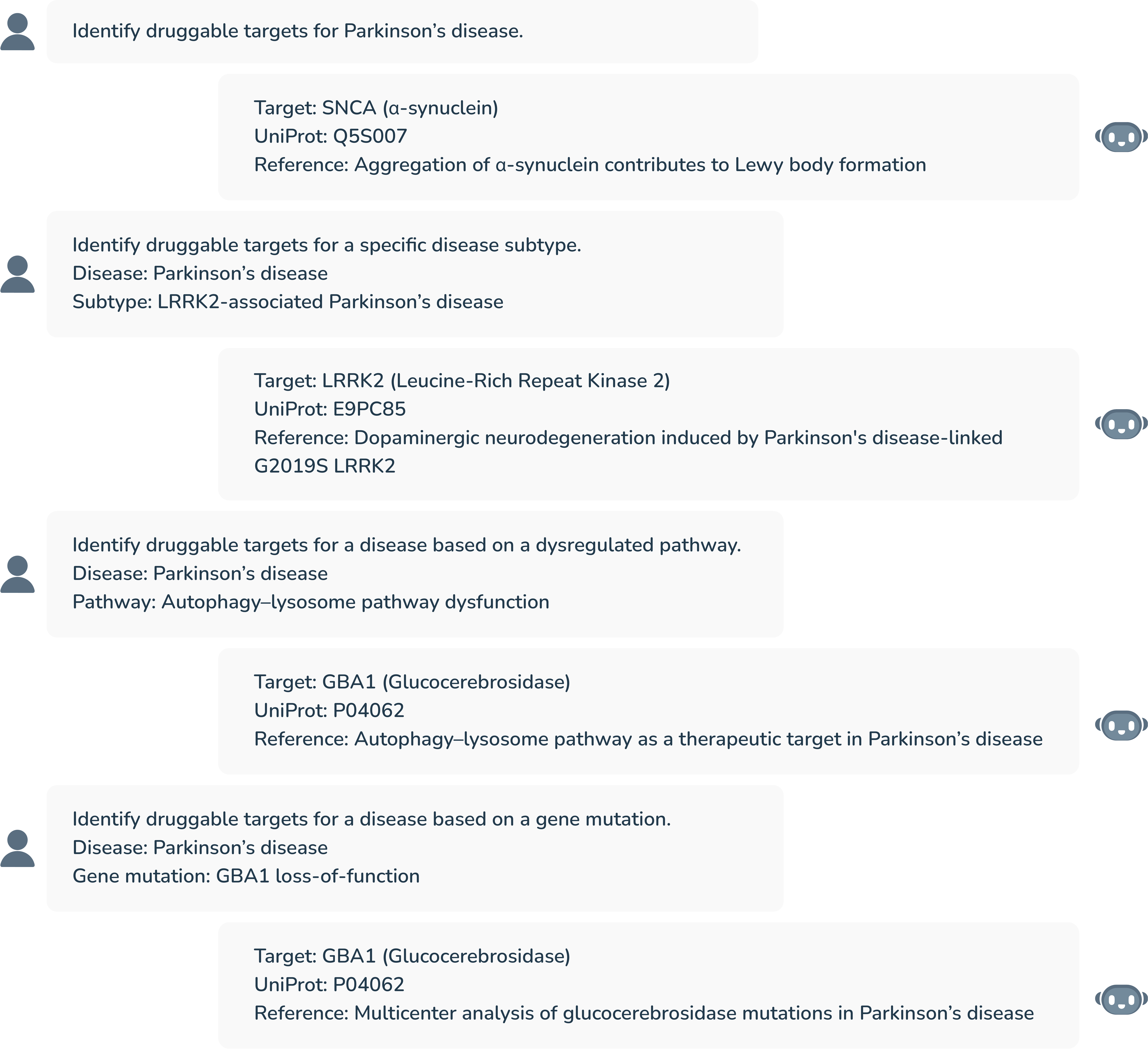}
\caption{Disease to biological target: the Literature Agent 
identifies the primary therapeutic protein target and its UniProt 
identifier from a disease name input.}
\label{fig:demo1}
\end{figure}

\begin{figure}[!hbt]
\centering
\includegraphics[width=\linewidth]{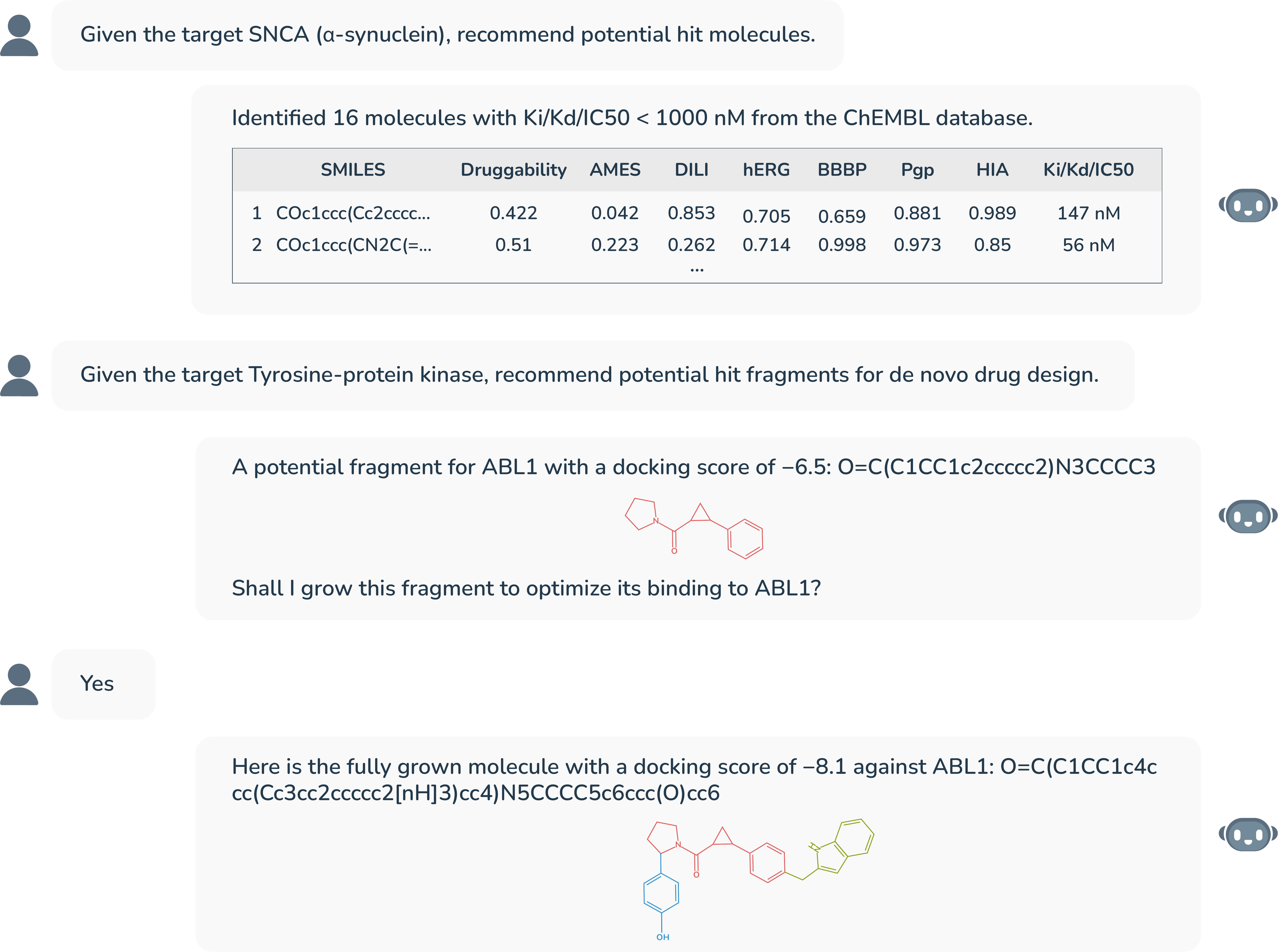}
\caption{Protein to lead: starting from a target protein, MolLingo 
retrieves known binders, clusters them for diversity, and performs 
docking-guided fragment growing to produce an optimized lead 
compound.}
\label{fig:demo2}
\end{figure}

\begin{figure}[!hbt]
\centering
\includegraphics[width=\linewidth]{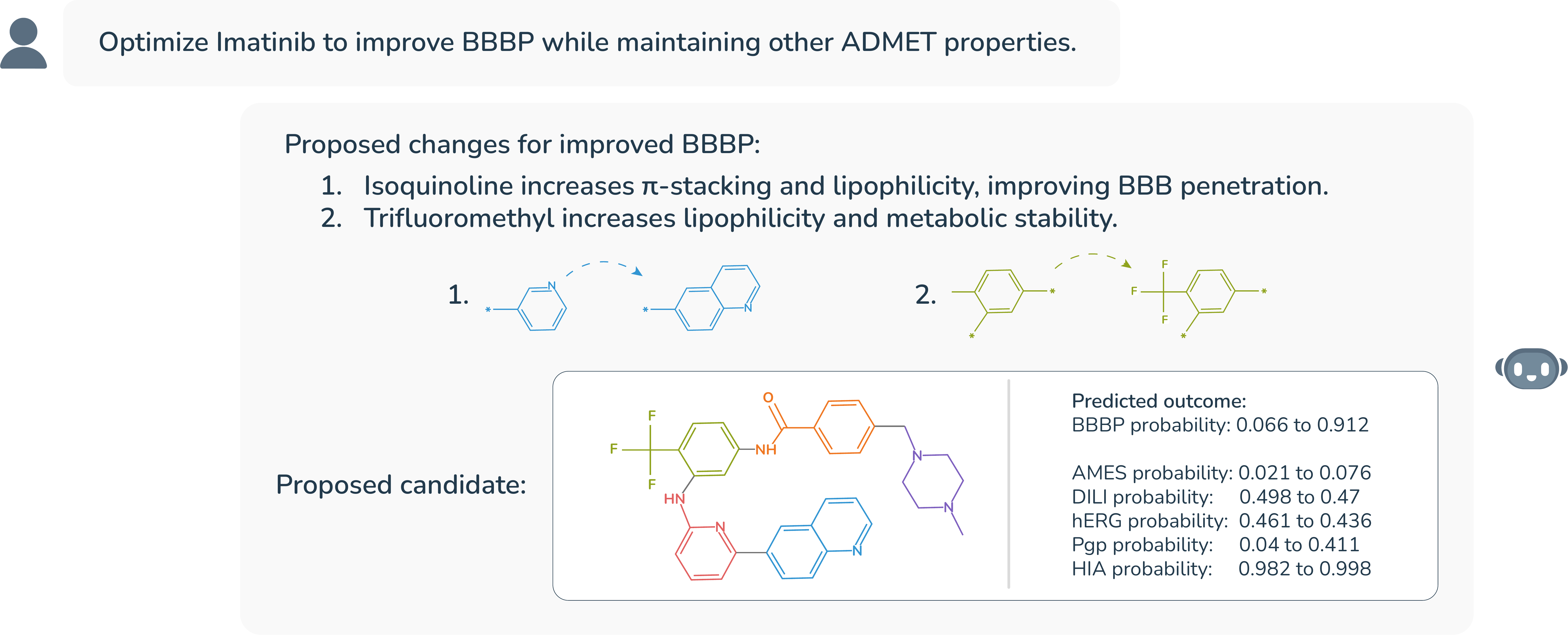}
\caption{Lead optimization: the Chemist Agent iteratively refines 
a lead molecule through block-level modifications guided by ADMET 
profiling and accumulated reasoning traces.}
\label{fig:demo3}
\end{figure}


\end{document}